\documentclass[11pt]{article}

\usepackage[T1]{fontenc}
\usepackage{lmodern}
\usepackage{microtype}
\usepackage[a4paper,margin=1in]{geometry}
\usepackage{amsmath,amssymb,amsthm,mathtools}
\usepackage{aliascnt}
\usepackage{bm}
\usepackage{booktabs}
\usepackage{array}
\usepackage{enumitem}
\usepackage[ruled,vlined,linesnumbered]{algorithm2e}
\usepackage[round,authoryear]{natbib}
\usepackage{xcolor}
\usepackage{url}
\usepackage{fancyhdr}
\usepackage{hyperref}
\usepackage[nameinlink,capitalize,noabbrev]{cleveref}

\hypersetup{
  colorlinks=true,
  linkcolor=blue!55!black,
  citecolor=green!40!black,
  urlcolor=blue!55!black,
  pdftitle={Balanced Exchanges: Polynomial-Time (1-1/e)-Regret for Bandit Submodular Maximization under Matroid Constraints},
  pdfauthor={Zongqi Wan and Zhijie Zhang},
  pdfsubject={Adversarial bandit submodular maximization under matroid constraints},
  pdfkeywords={submodular maximization, matroid, bandit feedback, online learning, Poisson process, fractional exchange}
}

\theoremstyle{plain}
\newtheorem{theorem}{Theorem}[section]
\newaliascnt{lemma}{theorem}
\newtheorem{lemma}[lemma]{Lemma}
\aliascntresetthe{lemma}
\newaliascnt{proposition}{theorem}
\newtheorem{proposition}[proposition]{Proposition}
\aliascntresetthe{proposition}
\newaliascnt{corollary}{theorem}

\aliascntresetthe{corollary}
\theoremstyle{definition}
\newaliascnt{definition}{theorem}
\newtheorem{definition}[definition]{Definition}
\aliascntresetthe{definition}
\theoremstyle{remark}
\newaliascnt{remark}{theorem}
\newtheorem{remark}[remark]{Remark}
\aliascntresetthe{remark}

\newif\ifomitproofs
\omitproofsfalse
\ifomitproofs
  \RenewDocumentEnvironment{proof}{o +b}{}{}
\fi

\crefname{algocf}{algorithm}{algorithms}
\Crefname{algocf}{Algorithm}{Algorithms}
\crefname{appsec}{appendix}{appendices}
\Crefname{appsec}{Appendix}{Appendices}
\crefname{theorem}{theorem}{theorems}
\Crefname{theorem}{Theorem}{Theorems}
\crefname{lemma}{lemma}{lemmas}
\Crefname{lemma}{Lemma}{Lemmas}
\crefname{proposition}{proposition}{propositions}
\Crefname{proposition}{Proposition}{Propositions}
\crefname{corollary}{corollary}{corollaries}
\Crefname{corollary}{Corollary}{Corollaries}
\crefname{definition}{definition}{definitions}
\Crefname{definition}{Definition}{Definitions}
\crefname{remark}{remark}{remarks}
\Crefname{remark}{Remark}{Remarks}

\allowdisplaybreaks
\numberwithin{equation}{section}
\setlist[itemize]{leftmargin=1.55em,itemsep=0.2em,topsep=0.3em}
\setlist[enumerate]{leftmargin=1.75em,itemsep=0.2em,topsep=0.3em}
\newcommand{\U}{\mathcal U}
\newcommand{\I}{\mathcal I}
\newcommand{\B}{\mathcal B}
\newcommand{\M}{\mathcal M}
\newcommand{\E}{\mathbb E}
\newcommand{\Prb}{\mathbb P}
\newcommand{\1}{\mathbf 1}
\newcommand{\eps}{\varepsilon}
\newcommand{\ind}{\mathbf 1}

\newcommand{\argmax}{\operatorname*{arg\,max}}
\newcommand{\argmin}{\operatorname*{arg\,min}}

\newcommand{\cH}{\mathcal H}
\newcommand{\cF}{\mathcal F}
\newcommand{\cE}{\mathcal E}

\newcommand{\wt}{\widetilde}
\newcommand{\wh}{\widehat}
\newcommand{\event}{J}
\newcommand{\ip}[2]{\left\langle #1,#2\right\rangle}

\newcommand{\cset}[1]{\left\{#1\right\}}

\newcommand{\deq}{\triangleq}

\title{Bandit Submodular Maximization under Matroid Constraints: Learning Compressed Exchange Policy}

\author{
\begin{minipage}[t]{0.4\textwidth}
\centering
Zongqi Wan\\[0.25em]
\small Great Bay University\\
\small \href{mailto:zqwan@gbu.edu.cn}{\texttt{zqwan@gbu.edu.cn}}
\end{minipage}
\hfill
\begin{minipage}[t]{0.4\textwidth}
\centering
Zhijie Zhang\\[0.25em]
\small Fuzhou University \\
\small \href{mailto:zzhang@fzu.edu.cn}{\texttt{zzhang@fzu.edu.cn}}
\end{minipage}
}
\date{ }

\begin{document}
\maketitle

\begin{abstract}
We study adversarial bandit maximization of monotone submodular functions under a matroid constraint.  For a rank-$k$ matroid on $n$ elements, we give a randomized oracle-polynomial algorithm that makes one feasible value query per round and has expected $(1-1/e)$-regret $\widetilde O(n^{1/3}k^{2/3}T^{2/3})$.  This is the first sublinear-regret algorithm for adversarial bandit submodular maximization under general matroid constraints.

Technically, we view the problem as learning an exchange policy for the Poisson base walk.  This connects the problem to contextual bandits and gives an information-theoretic sublinear-regret guarantee, but directly learning the exponentially many policies requires exponential time and space.  We therefore introduce \emph{balanced fractional exchanges}, which compress the policy mixture into a single fractional base while retaining the exchange information needed by the Poisson analysis.  This leads to an polynomial time algorithm with the same regret guarantee.
\end{abstract}

\tableofcontents
\newpage

\section{Introduction}
\label{sec:intro}

Selecting a collection of items whose joint value exhibits diminishing returns is a basic problem in combinatorial optimization.  The value of adding a new item is often largest when few related items have already been selected: an additional seed user reaches fewer new individuals after many influential users have been chosen, and an additional sensor supplies less information once nearby locations are already monitored.  Monotone submodular functions formalize precisely this phenomenon.  They underlie, among many other examples, influence maximization in social networks \citep{kempe2003} and information-based sensor placement \citep{krause2008}.  Matroid constraints complement the objective by expressing structured feasibility.  Uniform matroids encode cardinality budgets, partition matroids encode group or position quotas, and graphic and linear matroids impose global independence requirements.  Thus, monotone submodular maximization over a matroid provides a common abstraction for selecting a valuable yet nonredundant feasible collection.

The offline submodular maximization problem has a long history.  For a cardinality constraint, the classical greedy algorithm achieves the factor $1-1/e$ \citep{nemhauser1978}.  Under a general matroid constraint, the elementary greedy algorithm no longer attains this factor; the optimal polynomial-time guarantee was recovered by continuous greedy, together with the multilinear extension and dependent rounding \citep{vondrak2008,calinescu2011,chekuri2010}.  Non-oblivious local search provides a discrete alternative with the same approximation factor \citep{filmusward2014,buchbinderfeldman2024}.  Most recently, \citet{ganz2026poisson} obtained $1-1/e$ through a different nonhomogeneous Poisson walk approach on matroid bases.  Unlike continuous-relaxation algorithms, this walk moves only by feasible single-element exchanges, a feature that will be very helpful in this paper's setting.

In many applications, however, the objective is neither fixed nor known when the decision is made.  A platform repeatedly chooses a slate of recommendations or advertisements before observing user response; a campaign selects seed users before the resulting diffusion is realized; and a monitoring system deploys a feasible sensor configuration before its utility is revealed.  These examples motivate an online model in which, on each round $t$, the learner selects an independent set $S_t$ before seeing the current monotone submodular reward function $f_t$.  Under full-information feedback the learner subsequently receives access to the reward function, whereas under full-bandit feedback it observes only the single scalar $f_t(S_t)$.  The natural benchmark is the best fixed feasible set in hindsight, weakened by the approximation factor that is unavoidable even offline.

Online and bandit submodular maximization were first studied through online-greedy and meta-action reductions.  For cardinality constraints, \citet{streeter2008} obtained $O(T^{2/3})$ $(1-1/e)$-regret; \citet{streeter2009} developed the assignment setting, which corresponds to a structured partition-matroid constraint, with $O(T^{4/5})$ $(1-1/e)$-regret.  Offline-to-online reductions based on Blackwell approachability later gave a general explanation of how robust greedy algorithms can be converted into online algorithms and bandit analogues \citep{niazadeh2023}.  Stochastic full-bandit variants have also been studied under cardinality, knapsack, and general matroid constraints \citep{nie2022,nie2023framework,fourati2024stochasticgreedy,nie2025ksubmodular}.

Turning an offline $(1-1/e)$-approximation into a bandit algorithm faces a basic obstacle: most existing offline methods evaluate auxiliary subsets that need not satisfy the matroid constraint.  Such queries are harmless in the offline value-oracle model, where the algorithm may inspect the objective on any subset.  In the bandit model, however, the queried set is also the action played by the learner.  Under the original, or \emph{strict feasible-query}, model, every such set---including an exploratory set used only to construct an estimator---must be independent.

\citet{zhang2019} made this obstruction precise for the multilinear-extension approach.  They showed that, for some matroid, there is no function-independent randomized rounding scheme supported on feasible sets whose single observed value is an unbiased estimate of the multilinear extension for every submodular objective.  Motivated by this impossibility, they introduced \emph{responsive bandit submodular maximization}: the learner may query an infeasible set and observe its function value, but receives zero reward on that round.  Under this relaxed feedback model, they obtained $O(T^{8/9})$ $(1-1/e)$-regret.  
However, in the original adversarial bandit model, where every queried set must be feasible, no sublinear $(1-1/e)$-regret guarantee was known for arbitrary monotone submodular rewards under general matroid constraints.  Indeed, after establishing such a guarantee for partition matroids, \citet{wan2023bandit} conjectured that sublinear regret might fail for general matroid constraints.

We revisit this problem under the original strict feasible-query model. Our work is organized around the following question:
\begin{quote}
\emph{Can we design a bandit submodular maximization algorithm that achieves sublinear $(1-1/e)$-regret under general matroid constraints in polynomial time?}
\end{quote}
We answer this question affirmatively, thereby disproving the conjecture of \citet{wan2023bandit}.  Our approach builds on the recent Poisson-process method for offline submodular maximization \citep{ganz2026poisson}.  We first formulate its online counterpart as a policy learning problem: the current base of the Poisson walk is the context, a feasible exchange is an action, and each comparator base induces a state dependent exchange policy.  This contextual-bandit viewpoint, which has not appeared in prior work on this problem, gives a possible route to sublinear regret. However, this approach only produce a information theoretic guarantee, because the number of policies is exponentially large and require exponential time and space to excute. We therefore introduce a technique to compress the resulting policy mixture into a fractional base and couple it with each current base of the Poisson walk while preserving the required exchange marginals.  This new technique yields an polynomial time algorithm with sublinear $(1-1/e)$-regret under general matroid constraints.

\paragraph{Main result.} Our main result gives the optimal polynomial-time approximation factor $1-1/e$ in expected approximate regret, formally defined in \Cref{sec:prelim}. Concretely, we obtain the following theorem, which is an informal version of \Cref{thm:optimized-regret}.

\begin{theorem}
\label{thm:intro-main}
Let $T\ge2$, and let $\M$ be a rank-$k$ matroid on $n$ elements. There exists a randomized oracle-polynomial algorithm that makes exactly one feasible value query per round and satisfies
\[
\left(1-\frac1e\right)\max_{O\in\B}\sum_{t=1}^T f_t(O)
-\E\left[\sum_{t=1}^T f_t(S_t)\right]
=\widetilde O\!\left(n^{1/3}k^{2/3}T^{2/3}\right).
\]
\end{theorem}

To our knowledge, this is the first polynomial time sublinear $(1-1/e)$-regret guarantee for bandit monotone submodular maximization in the strict feasible-query bandit problem under general matroid constraint. The dependence on $T$ matches the $T^{2/3}$ scale attained for previous results on uniform matroid and partition matroid \citep{streeter2008,wan2023bandit}.

\subsection{Technical Overview}
\label{sec:technical-overview}

The starting point is the Poisson process approximation algorithm of \citet{ganz2026poisson}.  It evolves a feasible base through a continuous-time Poisson process, performing at each arrival a single-element exchange that satisfies a suitable local improvement condition.  These local conditions imply a differential inequality whose integration yields the $1-1/e$ approximation guarantee.  This form is particularly well suited to bandit feedback because the subset maintained by the algorithm can remain feasible throughout.  The difficulty is that an online learner must choose each exchange before observing the current reward function $f_t$ and hence cannot directly enforce the required local condition.  The Poisson analysis localizes the resulting loss: the gap in the final approximation is controlled by the accumulated deficits of the exchanges actually taken.  Our task is therefore to learn these exchanges from one feasible query per round.

\paragraph{Why the standard meta-framework does not apply.}
A standard way to onlineize a sequential approximation algorithm is to associate a no-regret linear optimizer with each update step: on every round, the optimizers generate the updates in sequence, and the optimizer at a given step is trained on the marginal objective induced by the preceding updates \citep{streeter2008,streeter2009,niazadeh2023}.  This template presumes that each step has a fixed action domain and that the offline compar ator can be represented by a fixed action for the corresponding optimizer.  The random number of Poisson arrivals can be handled by truncation, but the state dependence of an arrival cannot.  Indeed, at a given arrival, both the set of feasible exchanges and their gains depend on the current random base $A$.  Even for a fixed comparator base $O$, the exchange that certifies the Poisson improvement condition is selected through a Brualdi matching between $A$ and $O$, and therefore varies with $A$ rather than being a fixed pair $(i,j)$.  Thus, an online optimizer whose actions are individual exchanges has no fixed comparator valid across rounds.


\paragraph{First try: an information-theoretic policy learner.} 
To address the above barrier, an action must instead specify how to exchange from every possible current base, that is, it must be a state-dependent exchange policy. For every base $O$, fix a complete exchange policy $\pi_O$: whenever the Poisson walk is at a base $A$, the policy chooses a uniformly random element $i\in A$ and inserts the element matched to $i$ by a fixed Brualdi bijection from $A$ to $O$.  The policy class thus contains, for every possible hindsight base, a rule that is locally correct at every state that the walk may visit.  Treating the current base as the context and a feasible exchange as the action, one may run the \textsc{Exp4} algorithm over this finite policy class \citep{auer2002nonstochastic}.  With standard exploration, its finite-policy regret bound controls the accumulated Poisson deficits and yields sublinear approximation regret.  Although the policy class is exponentially large, the regret depends only logarithmically on its cardinality.  Computationally, however, \textsc{Exp4} maintains one weight per policy and therefore requires exponential time and memory.

\paragraph{Using balanced fractional exchange to compress the policy mixture.}
To resolve the computational issue, a key observation is that the Poisson process analysis does not require the identity of the mixed comparator. Let $W_t$ be the distribution over comparator policies in the above exponential time algorithm and set
\[
x_t
\deq
\sum_{O\in\B}W_t(O)\1_O
\in P(\M).
\]
When the Brualdi exchanges are averaged according to $W_t$, the resulting exchange distribution has two simple marginals: the element removed from the current base is uniform, and the element inserted has distribution $x_t/k$.  The Poisson analysis uses only these marginals, not the identity of the policy from which an exchange originated.  Thus the exponentially supported distribution $W_t$ can be replaced by the single fractional base $x_t$. We can therefore learn $x_t$ directly, without maintaining any distribution over bases.

The remaining question is whether these two marginals can be realized at every current base using only feasible exchanges.  This is precisely the role of a balanced fractional exchange.  Given $A$ and $x_t$, we construct a nonnegative matrix $Q_A^{x_t}$ supported on pairs $(i,j)$ for which $A-i+j$ is a base and satisfying
\[
\sum_{j\in\U}Q_A^{x_t}(i,j)=1
\quad(i\in A),
\qquad
\sum_{i\in A}Q_A^{x_t}(i,j)=x_{t,j}
\quad(j\in\U).
\]
Sampling an exchange from $Q_A^{x_t}/k$ therefore reproduces exactly the deletion and insertion marginals of the policy mixture.  We prove that, for every current base $A$ and every $x_t\in P(\M)$, such a matrix exists and can be constructed in polynomial time as a transportation flow on the bipartite exchange graph of $A$.

\paragraph{From balanced exchanges to a linear bandit.}
The marginal identities of a balanced exchange allow us to aggregate the state-dependent deficits along the Poisson trajectory without tracking the individual comparator-specific exchange rules.  Combining these identities with a local linearization, we upper-bound the cumulative Poisson deficit by the regret of an online linear optimization problem over the matroid base polytope: the learner chooses a fractional base $x_t$, and every hindsight base $O$ is represented by the fixed vertex $\1_O$.  The original bandit submodular problem is thereby reduced to a specialized linear bandit problem.  Its linear loss is not observed directly, but must be estimated from one feasible value query, which motivates the separate estimator construction below.

\paragraph{A feasible leave-one-out estimator.}
Constructing the estimator requires a further idea independent of the balanced-exchange representation.  The local improvement condition in the original Poisson algorithm is expressed through the gradient gain $\partial_jF(\tau\1_A)-\partial_iF(\tau\1_A)$.  Directly sampling its insertion term may require evaluating $f(S+j)$ for $S\sim\tau A$, which need not be feasible when $S$ contains the element $i$ that must be removed before inserting $j$.  The leave-one-out gain avoids this obstruction by sampling $R\sim\tau(A-i)$: both $R+j\subseteq A-i+j$ and $R+i\subseteq A$ are feasible.  Moreover, it equals the normalized multilinear jump caused by the exchange and dominates the original gradient gain, so it preserves the Poisson analysis.  After linearization leaves only an insertion loss, a fair-coin randomization between querying $R+j$ and $R$ gives a nonnegative unbiased estimate of this loss.  Thus the required feedback is obtained from one feasible value query.

The remaining online-learning step is standard: importance weighting turns the one-query statistic into an unbiased estimate of the linear loss, and entropic mirror descent controls its cumulative regret.  Together with a capped Poisson simulation and polynomial-time implementations of the balanced exchange flow and the entropy projection, we obtain the mentioned result.

\subsection{Related Work}

\paragraph{Offline submodular maximization.}
The cardinality-constrained greedy algorithm and its extensions initiated the approximation theory of monotone submodular maximization \citep{nemhauser1978}. Under a matroid constraint, continuous greedy and pipage or swap rounding achieve $1-1/e$ \citep{vondrak2008,calinescu2011,chekuri2010}. Non-oblivious local-search methods provide alternative proofs of the same factor \citep{filmusward2014,buchbinderfeldman2024}. The recent Poisson-process method of \citet{ganz2026poisson} realizes the approximation through a continuous-time base walk using few single-element exchanges.

\paragraph{Adversarial online and bandit submodular maximization.}
Early work developed online greedy and assignment algorithms for monotone submodular rewards \citep{streeter2008,streeter2009}. General offline-to-online frameworks transform robust greedy procedures through Blackwell approachability and its bandit analogue \citep{niazadeh2023}. Continuous DR-submodular methods obtain bandit guarantees through zeroth-order gradient estimation; their discrete responsive-feedback reduction may evaluate sets outside the original feasible family \citep{zhang2019}. For adversarial rewards under strict feasible feedback, \citet{wan2023bandit} obtain $(1-1/e)$-regret for partition matroids by exploiting a product-form feasible extension. \citet{sisalem2024} handle general matroids for the restricted class of weighted threshold-potential objectives through a concave relaxation, including a bandit variant. The present work treats arbitrary monotone submodular objectives and arbitrary matroids without a product decomposition or a concave-relaxation assumption.

\paragraph{Stochastic submodular bandits.}
Several stochastic models exploit additional structure or richer observations. Linear submodular bandits assume that the unknown utility is a linear combination of known submodular components \citep{yue2011}, while interactive submodular bandits impose an RKHS model and observe noisy marginal utilities \citep{chen2017interactive}. Closer to our feedback protocol is the stochastic full-bandit model, in which only the noisy total reward of the selected set is observed. Under a cardinality constraint, \citet{nie2022} give an explore-then-commit algorithm for monotone rewards, and \citet{fourati2024stochasticgreedy} improve the dependence on the cardinality bound through stochastic-greedy exploration; \citet{tajdini2024} subsequently derive the first minimax lower bounds together with a nearly matching upper bound for a restricted algorithm class. \citet{nie2023framework} give a black-box reduction from robust offline approximation algorithms to stochastic full-bandit algorithms, with applications to monotone submodular maximization under cardinality and knapsack constraints. For general matroids, specializing the monotone $k_{\mathrm{sub}}$-submodular result of \citet{nie2025ksubmodular} to $k_{\mathrm{sub}}=1$ yields ordinary set-submodular maximization with approximation factor $1/2$ and $\widetilde O(n^{1/3}kT^{2/3})$ expected $1/2$-regret. For unconstrained nonmonotone objectives that are submodular in expectation, \citet{fourati2023} obtain a stochastic full-bandit guarantee. These results rely on a fixed mean reward function or other stochastic structure, whereas our reward sequence may be chosen adversarially.

\paragraph{Bandit submodular minimization.}
For unconstrained submodular minimization, the convex Lov\'asz extension permits a substantially different reduction to online convex optimization. \citet{hazan2012} give efficient full-information and bandit algorithms, including an $O(nT^{2/3})$ adversarial bandit regret bound. \citet{cardoso2019} develop differentially private variants under both feedback models, and \citet{ito2022} obtain best-of-both-worlds guarantees with gap-dependent logarithmic regret in stochastic environments and robustness to adversarial corruptions. These convex-extension techniques do not directly apply to maximizing a submodular reward under a matroid constraint, where the objective is nonconcave and the relevant benchmark is approximate.


\subsection{Paper Organization}

\Cref{sec:prelim} introduces the model and notation.  \Cref{sec:poisson} extends the Poisson-walk algorithm of \citet{ganz2026poisson} to general exchange distributions and quantifies the resulting terminal approximation error.  \Cref{sec:balanced} establishes balanced fractional exchanges.  \Cref{sec:algorithm} presents the bandit algorithm and its regret analysis.  Supporting arguments appear in the appendices.

\section{Preliminaries}
\label{sec:prelim}

\paragraph{Submodular functions.}
For a finite ground set $\U$ and a set $S\subseteq\U$, let $\1_S\in\{0,1\}^{\U}$ denote the incidence vector of $S$. For an element $u\in\U$, write $S+u\deq S\cup\{u\}$ and $S-u\deq S\setminus\{u\}$. A function $f:2^{\U}\to[0,1]$ is \emph{normalized} if $f(\varnothing)=0$, \emph{monotone} if for any $A\subseteq B$ we have $f(A)\le f(B)$. We call a set function is \emph{submodular} if it satisfies following property for any $A\subseteq B,\ u\notin B$:
\[
f(A+u)-f(A)
\ge
f(B+u)-f(B).
\]

\paragraph{Matroids.}
Let $\M=(\U,\I)$ be a matroid of rank $k\ge1$. The family of bases is denoted by $\B$. For a base $A$, define its feasible single-exchange set
\begin{equation}
\cE(A)\deq\cset{(i,j): i\in A,\ j\in\U,\ A-i+j\in\B}.
\label{eq:exchange-set}
\end{equation}
Self-loops $(i,i)$ are included. If $j\in A$ and $(i,j)\in\cE(A)$, then necessarily $i=j$.

The base polytope is
\begin{equation}
P(\M)\deq\operatorname{conv}\cset{\1_B:B\in\B}.
\label{eq:base-polytope}
\end{equation}
Every $x\in P(\M)$ satisfies $0\le x_u\le1$ and $\sum_{u\in\U}x_u=k$. We need Brualdi's strong basis-exchange theorem \citep{brualdi1969} in our analysis.
\begin{lemma}[\citep{brualdi1969}]
\label{lem:brualdi}
For every pair of bases $A,B\in\B$, there is a bijection $h_{A,B}:A\to B$ such that
\[
A-i+h_{A,B}(i)\in\B
\qquad\text{for every }i\in A.
\]
The bijection may be chosen so that $h_{A,B}(i)=i$ for all $i\in A\cap B$.
\end{lemma}

\paragraph{Bandit protocol.}
An oblivious adversary fixes a sequence $f_1,\ldots,f_T:2^{\U}\to[0,1]$ of normalized, monotone submodular functions before the interaction. In each round $t=1,\ldots,T$, the learner chooses $S_t\in\I$ without observing $f_t$, then observes only the scalar $f_t(S_t)$ and receives this value. Because every independent set extends to a base and every $f_t$ is monotone,
\[
\max_{S\in\I}\sum_{t=1}^T f_t(S)
=
\max_{O\in\B}\sum_{t=1}^T f_t(O).
\]
Accordingly, for $\alpha\in[0,1]$, the expected $\alpha$-regret is
\[
R_\alpha(T)
\deq
\alpha\max_{O\in\B}\sum_{t=1}^T f_t(O)
-\E\left[\sum_{t=1}^T f_t(S_t)\right].
\]

\paragraph{Multilinear extension.}

Fix a monotone submodular function $f:2^{\U}\to[0,1]$. Its multilinear extension is
\begin{equation}
F(x)\deq\E_{R\sim x}[f(R)],
\label{eq:multilinear}
\end{equation}
where every element $u$ is included independently with probability $x_u$. For $S\subseteq\U$ and $\tau\in[0,1]$, the notation $R\sim\tau S$ means that every element of $S$ is included independently with probability $\tau$ and no element outside $S$ is included.

The following two standard properties of the multilinear extension are basic ingredients in the analysis of continuous greedy \citep{vondrak2008,calinescu2011}. We include their short proofs for completeness.

\begin{lemma}
\label{lem:multilinear-marginal}
For every $u\in\U$,
\begin{equation}
\partial_uF(x)
=
\E_{R\sim x|_{\U-u}}[f(R+u)-f(R)].
\label{eq:derivative}
\end{equation}
The derivative is independent of $x_u$ and is nonincreasing in every other coordinate. Consequently, if $x\le y$ coordinatewise, then
\[
\partial_uF(x)\ge\partial_uF(y).
\]
\end{lemma}

\begin{proof}
Because $F$ is multilinear, conditioning on all coordinates except $u$ gives
\[
F(x)=x_u\E[f(R+u)]+(1-x_u)\E[f(R)],
\]
where $R\sim x|_{\U-u}$. Differentiating in $x_u$ proves \eqref{eq:derivative} and shows independence from $x_u$.

Now let $x$ and $y$ agree on coordinate $u$ and satisfy $x_v\le y_v$ for every $v\ne u$. Couple $R_x\sim x|_{\U-u}$ and $R_y\sim y|_{\U-u}$ by using common uniform random variables, so that $R_x\subseteq R_y$ almost surely. Submodularity gives
\[
f(R_x+u)-f(R_x)
\ge
f(R_y+u)-f(R_y).
\]
Taking expectations proves monotonicity in every coordinate other than $u$. Independence from $x_u$ completes the proof for arbitrary $x\le y$.
\end{proof}

\begin{lemma}
\label{lem:multilinear-gap}
For every $x\in[0,1]^{\U}$ and every set $O\subseteq\U$,
\begin{equation}
\sum_{o\in O}\partial_oF(x)
\ge
f(O)-F(x).
\label{eq:multilinear-gap}
\end{equation}
\end{lemma}

\begin{proof}
Let $x\vee\1_O$ denote the coordinatewise maximum. Monotonicity gives
$F(x\vee\1_O)\ge F(\1_O)=f(O)$. Increase the coordinates in $O$ one at a time from their values in $x$ to one. By \Cref{lem:multilinear-marginal}, every marginal encountered along this monotone path is at most its value at $x$. Therefore
\[
F(x\vee\1_O)-F(x)
\le
\sum_{o\in O}(1-x_o)\partial_oF(x)
\le
\sum_{o\in O}\partial_oF(x).
\]
Combining the two displays proves the claim.
\end{proof}

\section{Poisson Base Walk with General Exchange Distribution}
\label{sec:poisson}

The Poisson-process algorithm of \citet{ganz2026poisson} is an important component of our bandit algorithm. We formulate the base walk with a general distribution over feasible exchanges, which is the form used throughout this paper, and then recover their offline algorithm as a particular exchange rule. For simplicity, we call the algorithm of \citet{ganz2026poisson} the GKSS algorithm, after its authors.
The distinction lies in the exchange rule: the GKSS algorithm can choose an exchange using the entire objective $f$, whereas the bandit learner must choose one before observing $f_t$.

\paragraph{Poisson base walk.}
Fix a monotone submodular function $f$ with multilinear extension $F$, a rank-$k$ matroid $\M$, a starting time $\eps\in(0,1)$, and an arbitrary starting base $A_0$. The walk runs a nonhomogeneous Poisson process on $[\eps,1]$ with intensity $\lambda(s)\deq k/s$. For every time $s$ and base $A$, let $p_s(\cdot\mid A)$ be a measurable distribution supported on the feasible exchanges $\cE(A)$. We call each arrival of this Poisson process an \emph{event}: event $\event$ occurs at time $\tau_{\event}$, samples one feasible exchange $e_{\event}=(i_{\event},j_{\event})$, and updates the base from $A_{\event-1}$ to $A_{\event}$. Thus an event is an internal jump of the simulated base trajectory, not a value query.

\begin{algorithm}[t]
\small
\caption{Poisson Base Walk with General Exchange Distributions}
\label{alg:poisson-base-walk}
\KwIn{Matroid $\M$; starting base $A_0$; starting time $\eps$; exchange distributions $p_s(\cdot\mid A)$ supported on $\cE(A)$.}
Sample the event times $\eps<\tau_1<\cdots<\tau_M\le1$ of a nonhomogeneous Poisson process with intensity $\lambda(s)=k/s$.\;
\For{$\event=1,\ldots,M$}{
 Draw $e_{\event}=(i_{\event},j_{\event})\sim p_{\tau_{\event}}(\cdot\mid A_{\event-1})$.\;
 Set $A_{\event}=A_{\event-1}-i_{\event}+j_{\event}$.\;
}
\Return $A_M$.\;
\end{algorithm}

Let $A(s-)$ denote the base immediately before time $s$, and let $A(s)$ denote the right-continuous post-event state. In particular, $A(\tau_{\event}-)=A_{\event-1}$ and $A(\tau_{\event})=A_{\event}$. Every state in \Cref{alg:poisson-base-walk} is a base, so the process requires no rounding. Moreover, the number $M$ of exchange events satisfies
\begin{equation}
M\sim\operatorname{Poisson}(\mu),
\qquad
\mu\deq\int_\eps^1\frac{k}{s}\,ds
=
k\log(1/\eps),
\label{eq:poisson-mean}
\end{equation}
and hence the expected number of exchanges is $k\log(1/\eps)$.  Conditional on $M$, the event times can be sampled directly as described in \Cref{app:poisson-sampling}.

The GKSS algorithm is the objective-dependent specialization of \Cref{alg:poisson-base-walk} obtained as follows. At an event time $s$ with current base $A$, it computes a maximum-weight base
\begin{equation}
B_s(A)
\in
\operatorname*{arg\,max}_{B\in\B}
\sum_{j\in B}\partial_jF(s\1_A).
\label{eq:gkss-best-base}
\end{equation}
By Brualdi's basis-exchange theorem, there is a bijection $h_{A,B_s(A)}:A\to B_s(A)$ such that $A-i+h_{A,B_s(A)}(i)$ is a base for every $i\in A$. The GKSS exchange distribution chooses $i$ uniformly from $A$ and sets $j=h_{A,B_s(A)}(i)$. Conditional on $A$ and $s$, this exchange satisfies, for every comparator base $O$,
\begin{align}
&\E\left[
\partial_jF(s\1_A)-\partial_iF(s\1_A)
\,\middle|\,
A,s
\right]
\notag=
\frac1k\left(
\sum_{j\in B_s(A)}\partial_jF(s\1_A)
-
\sum_{i\in A}\partial_iF(s\1_A)
\right)
\notag\\
&\quad\ge
\frac1k\left(
\sum_{o\in O}\partial_oF(s\1_A)
-
\sum_{i\in A}\partial_iF(s\1_A)
\right).
\label{eq:gkss-above-average}
\end{align}
The equality follows because both the deletion element and its image under the Brualdi bijection are uniform on their respective bases; the inequality follows from the definition of $B_s(A)$. The GKSS rule satisfies this condition at every event time, which is crucial for achieving the $1-1/e$ approximation factor.

The bandit algorithm in \Cref{sec:algorithm} instantiates \Cref{alg:poisson-base-walk} with a different exchange distribution. Because that distribution is chosen without observing the current reward function, it cannot in general satisfy \eqref{eq:gkss-above-average} at every event. The remainder of this section bounds the terminal value when the expected exchange gain falls below the right-hand side of \eqref{eq:gkss-above-average}.


\paragraph{Leave-One-Out exchange gain.}
For a feasible exchange $(i,j)\in\cE(A)$, define the \emph{leave-one-out exchange gain} for a given submodular function $f$ and its multilinear extension $F$, a time $\tau\in [0,1]$ as
\begin{equation}
G_f(A,\tau;i,j)
\deq
\E_{R\sim\tau(A-i)}[f(R+j)-f(R+i)].
\label{eq:loo-gain}
\end{equation}
The two sets in this difference are feasible: $R+i\subseteq A$ and $R+j\subseteq A-i+j$.

We have following two properties of $G_f$ that are used in the analysis of the Poisson process algorithm.

\begin{lemma}
\label{lem:exact-jump}
For every $(i,j)\in\cE(A)$,
\begin{equation}
F(\tau\1_{A-i+j})-F(\tau\1_A)
=
\tau G_f(A,\tau;i,j).
\label{eq:exact-jump}
\end{equation}
\end{lemma}

\begin{proof}
Both vectors in \eqref{eq:exact-jump} agree outside coordinates $i$ and $j$. Insert the intermediate vector $\tau\1_{A-i}$. Multilinearity in coordinate $j$ gives
\[
F(\tau\1_{A-i+j})-F(\tau\1_{A-i})
=
\tau\partial_jF(\tau\1_{A-i}),
\]
while multilinearity in coordinate $i$ gives
\[
F(\tau\1_A)-F(\tau\1_{A-i})
=
\tau\partial_iF(\tau\1_{A-i}).
\]
Subtracting the second identity from the first and using \eqref{eq:loo-gain} proves the claim.
\end{proof}

\begin{lemma}
\label{lem:loo-gradient-dominance}

For every base $A$, time $\tau\in[0,1]$, and feasible exchange $(i,j)\in\cE(A)$,
\begin{equation}
G_f(A,\tau;i,j)
\ge
\partial_jF(\tau\1_A)-\partial_iF(\tau\1_A).
\label{eq:G-gradient-dominance}
\end{equation}
\end{lemma}

\begin{proof}
The definition of $G_f$ and \Cref{lem:multilinear-marginal} give, for every feasible exchange $(i,j)$,
\begin{align*}
G_f(A,\tau;i,j)
&=
\partial_jF(\tau\1_{A-i})
-
\partial_iF(\tau\1_{A-i})\\
&\ge
\partial_jF(\tau\1_A)
-
\partial_iF(\tau\1_A),
\end{align*}
where the derivative in coordinate $i$ is unchanged when coordinate $i$ itself is varied.
\end{proof}



We first give the event-sum identity needed below. Its proof is included because the discount factor in the regret argument comes directly from the cancellation between the intensity $k/s$ and a factor $s$ in the event statistic.

\begin{lemma}
\label{lem:event-sum}
Let $H$ be any bounded event statistic, where $H(A,s,e)$ assigns a real value to a base $A$, a time $s\in[\eps,1]$, and a feasible exchange $e\in\cE(A)$.
Assume that $H(A,\cdot,e)$ is measurable and continuous except at finitely many points. Then
\begin{align}
\E\left[
\sum_{\event=1}^{M}
H(A(\tau_{\event}-),\tau_{\event},e_{\event})
\right]
=
\int_\eps^1\frac{k}{s}
\E\left[
\sum_{e\in\cE(A(s-))}
 p_s(e\mid A(s-))H(A(s-),s,e)
\right]ds,
\label{eq:event-sum}
\end{align}
where $A(s-)$ is the base held immediately before time $s$, so $A(\tau_{\event}-)$ is the state immediately before the $\event$th exchange.
\end{lemma}

\begin{proof}
Let $K$ bound $|H|$, and note that $\lambda(s)\le k/\eps$ on $[\eps,1]$. Consider a partition
$\eps=t_0<t_1<\cdots<t_m=1$ with mesh $\Delta$, and a cell $I_r=(t_{r-1},t_r]$. Conditional on the history at $t_{r-1}$, the number of underlying Poisson events in $I_r$ has mean
\[
\Lambda_r=\int_{t_{r-1}}^{t_r}\lambda(s)\,ds=O(\Delta).
\]
The probability of two or more events in this cell is
\[
1-e^{-\Lambda_r}(1+\Lambda_r)=O(\Delta^2),
\]
uniformly in $r$. Thus the total expected contribution from cells containing at least two events is at most
\[
K\sum_{r=1}^m\E[N(I_r)\ind\{N(I_r)\ge2\}]
=O(m\Delta^2)=O(\Delta),
\]
where the implicit constant depends only on $K,k,$ and $\eps$.

On a cell containing exactly one event at time $s$, the pre-event base equals the base at the left endpoint, and the conditional mark distribution is $p_s(\cdot\mid A(t_{r-1}))$. The standard density of the unique event in a nonhomogeneous Poisson process gives, conditionally on the left-endpoint history,
\begin{align*}
&\E\left[
\ind\{N(I_r)=1\}
H(A(s-),s,e)
\,
\middle|
\cF_{t_{r-1}}
\right]\\
&\quad=
\int_{I_r}
\exp\!\left(-\int_{t_{r-1}}^{t_r}\lambda(u)\,du\right)
\lambda(s)
\sum_e p_s(e\mid A(t_{r-1}))H(A(t_{r-1}),s,e)
\,ds.
\end{align*}
The exponential factor is $1+O(\Delta)$ uniformly. Replacing $A(t_{r-1})$ by the actual pre-event state inside the integral changes only paths with an earlier event in the same cell, whose total contribution is already $O(\Delta)$ after summing over cells. Therefore, summing over $r$ gives
\begin{align*}
\E\left[\sum_{\event=1}^M H(A(\tau_{\event}-),\tau_{\event},e_{\event})\right]
=
\sum_{r=1}^m\int_{I_r}\lambda(s)
\E\left[
\sum_e p_s(e\mid A(s-))H(A(s-),s,e)
\right]ds+O(\Delta).
\end{align*}
The integrand is bounded by $K\lambda(s)$ and is measurable. Letting the mesh tend to zero and applying dominated convergence proves \eqref{eq:event-sum}. A bounded measurable $H$ can be obtained by the usual simple-function approximation; the stated regularity already covers every function used in this paper.
\end{proof}

For a comparator base $O$, define the local benchmark
\begin{equation}
\zeta_{f,O}(A,s)
\deq
\frac1k\left(
\sum_{o\in O}\partial_oF(s\1_A)
-
\sum_{a\in A}\partial_aF(s\1_A)
\right),
\label{eq:zeta}
\end{equation}
and the discount weight
\begin{equation}
w(s)\deq s\exp(-(1-s)).
\label{eq:discount-weight}
\end{equation}

\begin{lemma}
\label{lem:poisson-generator}
Let
\begin{equation}
q(s)\deq\E[F(s\1_{A(s)})].
\label{eq:q-def}
\end{equation}
Then $q$ is absolutely continuous and, for almost every $s\in[\eps,1]$,
\begin{equation}
q'(s)
=
\E\left[
\sum_{a\in A(s-)}\partial_aF(s\1_{A(s-)})
+
k\sum_{e\in\cE(A(s-))}p_s(e\mid A(s-))G_f(A(s-),s;e)
\right].
\label{eq:poisson-generator}
\end{equation}
\end{lemma}

\begin{proof}
Fix $r\in[\eps,1]$. Along every sample path, the map $s\mapsto F(s\1_{A(s)})$ changes continuously between event times and has finitely many jumps. Holding the base $A$ fixed,
\begin{equation}
\frac{d}{ds}F(s\1_A)
=
\sum_{a\in A}\partial_aF(s\1_A).
\label{eq:continuous-drift}
\end{equation}
At an event with pre-event base $A$ and mark $e=(i,j)$, \Cref{lem:exact-jump} gives the jump
\begin{equation}
F(s\1_{A-i+j})-F(s\1_A)
=
sG_f(A,s;e).
\label{eq:jump-size}
\end{equation}
The pathwise decomposition is therefore
\begin{align}
F(r\1_{A(r)})-F(\eps\1_{A_0})
&=
\int_\eps^r
\sum_{a\in A(s-)}\partial_aF(s\1_{A(s-)})\,ds
\notag\\
&\quad+
\sum_{\event:\tau_{\event}\le r}
\tau_{\event} G_f(A(\tau_{\event}-),\tau_{\event};e_{\event}).
\label{eq:path-decomposition}
\end{align}
All terms are integrable: derivatives and gains lie in $[-1,1]$, and the expected event count is finite. Taking expectations and applying \Cref{lem:event-sum} to
$H(A,s,e)\deq sG_f(A,s;e)$ transforms the expected jump sum into
\[
\int_\eps^r
\frac{k}{s}\E\left[
\sum_ep_s(e\mid A(s-))sG_f(A(s-),s;e)
\right]ds.
\]
Substituting into \eqref{eq:path-decomposition} expresses $q(r)-q(\eps)$ as the integral of the right-hand side of \eqref{eq:poisson-generator}. Hence $q$ is absolutely continuous and has that derivative almost everywhere.
\end{proof}

\begin{theorem}
\label{thm:general-poisson}
Fix a comparator base $O$. For each event $\event$, define the discounted deficit
\begin{equation}
D_{\event}(O)
\deq
w(\tau_{\event})
\left(
\zeta_{f,O}(A(\tau_{\event}-),\tau_{\event})
-
G_f(A(\tau_{\event}-),\tau_{\event};e_{\event})
\right).
\label{eq:event-deficit}
\end{equation}
Then
\begin{align}
\E[f(A(1))]
&\ge
\bigl(1-e^{-(1-\eps)}\bigr)f(O)
-
\E\left[\sum_{\event=1}^M D_{\event}(O)\right]
\label{eq:poisson-exact-coefficient}\\
&\ge
(1-\eps)(1-1/e)f(O)
-
\E\left[\sum_{\event=1}^M D_{\event}(O)\right].
\label{eq:poisson-relaxed-coefficient}
\end{align}
\end{theorem}

\begin{proof}
Define the instantaneous expected deficit
\begin{align}
\delta(s)
\deq
\E\Bigg[
\zeta_{f,O}(A(s-),s)
-
\sum_{e\in\cE(A(s-))}p_s(e\mid A(s-))G_f(A(s-),s;e)
\Bigg].
\label{eq:instantaneous-deficit}
\end{align}
Use \eqref{eq:zeta} to rewrite the generator \eqref{eq:poisson-generator}:
\begin{align*}
q'(s)
&=
\E\left[
\sum_{a\in A(s-)}\partial_aF(s\1_{A(s-)})
+
k\zeta_{f,O}(A(s-),s)
\right]
-k\delta(s)\\
&=
\E\left[
\sum_{o\in O}\partial_oF(s\1_{A(s-)})
\right]
-k\delta(s).
\end{align*}
Applying \Cref{lem:multilinear-gap} at $s\1_{A(s-)}$ and taking expectations gives
\[
\E\left[\sum_{o\in O}\partial_oF(s\1_{A(s-)})\right]
\ge
f(O)-q(s).
\]
Therefore, for almost every $s$,
\begin{equation}
q'(s)+q(s)
\ge
f(O)-k\delta(s).
\label{eq:poisson-ode}
\end{equation}
Multiply by $e^s$ and integrate from $\eps$ to $1$:
\begin{align*}
e q(1)-e^\eps q(\eps)
&\ge
f(O)(e-e^\eps)
-
\int_\eps^1e^s k\delta(s)\,ds.
\end{align*}
Since $q(\eps)\ge0$, division by $e$ yields
\begin{equation}
q(1)
\ge
\bigl(1-e^{-(1-\eps)}\bigr)f(O)
-
\int_\eps^1e^{-(1-s)}k\delta(s)\,ds.
\label{eq:ode-integrated}
\end{equation}

Apply \Cref{lem:event-sum} with
\[
H(A,s,e)
\deq
s e^{-(1-s)}
\bigl(\zeta_{f,O}(A,s)-G_f(A,s;e)\bigr).
\]
The factor $s$ cancels the intensity $k/s$, so the expected event sum equals the integral in \eqref{eq:ode-integrated}. Also $q(1)=\E[F(\1_{A(1)})]=\E[f(A(1))]$. This proves \eqref{eq:poisson-exact-coefficient}.

Finally, the function $a\mapsto1-e^{-a}$ is concave on $[0,1]$ and vanishes at zero. Hence
\[
1-e^{-(1-\eps)}
\ge
(1-\eps)(1-1/e),
\]
which proves \eqref{eq:poisson-relaxed-coefficient}.
\end{proof}

The role of the deficit term distinguishes \Cref{thm:general-poisson} from the analysis of \citet{ganz2026poisson}. Under their objective-dependent exchange rule, the conditional expected deficit at every event is nonpositive, so the expected cumulative deficit is nonpositive and can be discarded from the lower bound. This gives the coefficient $1-e^{-(1-\eps)}$, which tends to $1-1/e$ as $\eps\downarrow0$. In contrast, \Cref{alg:poisson-base-walk} permits a general exchange distribution $p_s(\cdot\mid A)$, which need not satisfy the exchange gain condition of GKSS. Our theorem therefore retains the cumulative deficit explicitly and quantifies exactly how its expected value degrades the approximation guarantee. In the online setting, bounding this term converts the learner's cumulative local error into approximation regret.



\section{Balanced Fractional Exchanges}
\label{sec:balanced}

The discussion in \Cref{sec:technical-overview} suggests that an exponential mixture of comparator-specific Brualdi rules contains more information than the Poisson analysis needs. Its uniform deletion marginal and fractional insertion marginal can instead be represented by a single point $x\in P(\M)$. This section constructs a distribution over feasible exchanges with exactly these two marginals and shows that it can be computed by a polynomial-size transportation problem. The next section uses the marginal identities to derive the linear loss needed by the bandit algorithm.


\begin{definition}[Balanced fractional exchange]
\label{def:balanced-exchange}
Fix a base $A\in\B$ and a point $x\in P(\M)$. A \emph{balanced fractional exchange} from $A$ with insertion marginal $x$ is a nonnegative matrix
\[
Q_A^x(i,j),\qquad i\in A,\ j\in\U,
\]
satisfying
\begin{align}
Q_A^x(i,j)>0
&\Longrightarrow A-i+j\in\B,
\label{eq:balanced-support}\\
\sum_{j\in\U}Q_A^x(i,j)&=1
&&\text{for every }i\in A,
\label{eq:balanced-row}\\
\sum_{i\in A}Q_A^x(i,j)&=x_j
&&\text{for every }j\in\U.
\label{eq:balanced-column}
\end{align}
The corresponding probability distribution on feasible exchanges is
\begin{equation}
q_A^x(i,j)\deq\frac1k Q_A^x(i,j).
\label{eq:balanced-distribution}
\end{equation}
\end{definition}

The row and column sums both have total mass $k$, so \eqref{eq:balanced-distribution} is a probability distribution. Its two marginals are
\begin{equation}
\Prb_{(i,j)\sim q_A^x}(i=a)=\frac1k
\quad(a\in A),
\qquad
\Prb_{(i,j)\sim q_A^x}(j=u)=\frac{x_u}{k}
\quad(u\in\U).
\label{eq:balanced-marginals}
\end{equation}

The following theorem shows that balanced fractional exchanges exist for every base and every fractional base. 
\begin{theorem}
\label{thm:balanced-existence}
For every base $A\in\B$ and every fractional base $x\in P(\M)$, a balanced fractional exchange $Q_A^x$ exists.
\end{theorem}

\begin{proof}
By the definition of the base polytope, there are coefficients $\lambda_B\ge0$ with $\sum_{B\in\B}\lambda_B=1$ such that
\begin{equation}
x=\sum_{B\in\B}\lambda_B\1_B.
\label{eq:base-decomposition}
\end{equation}
For every $B$ with $\lambda_B>0$, choose a Brualdi bijection $h_{A,B}:A\to B$ as in \Cref{lem:brualdi}, and define
\begin{equation}
Q_A^x(i,j)
\deq
\sum_{B\in\B}\lambda_B\ind\{h_{A,B}(i)=j\}.
\label{eq:Q-decomposition}
\end{equation}
Every positive summand in \eqref{eq:Q-decomposition} is supported on a feasible exchange.  For a fixed $i\in A$, each bijection assigns exactly one image to $i$, so
\[
\sum_{j\in\U}Q_A^x(i,j)
=
\sum_{B\in\B}\lambda_B
=1.
\]
For a fixed $j\in\U$, bijectivity implies that exactly one $i\in A$ satisfies $h_{A,B}(i)=j$ when $j\in B$, and none does when $j\notin B$.  Hence
\[
\sum_{i\in A}Q_A^x(i,j)
=
\sum_{B\in\B}\lambda_B\ind\{j\in B\}
=x_j,
\]
where the last equality follows from \eqref{eq:base-decomposition}.  Thus all conditions in \Cref{def:balanced-exchange} hold.
\end{proof}

The preceding proof is existential because it begins with a base decomposition. A direct transportation formulation gives the desired oracle-polynomial construction.

\begin{theorem}
\label{thm:balanced-flow}
Fix a base $A$ and a fractional base $x\in P(\M)$. Given $x$ explicitly and access to an independence oracle for $\M$, a balanced fractional exchange $Q_A^x$ can be computed by a transportation flow on a bipartite graph with $k+n$ nonterminal vertices and at most $kn$ edges.
\end{theorem}

\begin{proof}
Construct a directed network with a source $s$, a sink $t$, a left vertex for every $i\in A$, and a right vertex for every $j\in\U$. Add the following arcs:
\begin{itemize}
\item an arc $s\to i$ of capacity exactly $1$ for each $i\in A$;
\item an arc $i\to j$ of unbounded capacity whenever $A-i+j\in\B$;
\item an arc $j\to t$ of capacity exactly $x_j$ for each $j\in\U$.
\end{itemize}
Here ``capacity exactly'' means an equal lower and upper bound. The support graph is found with at most $kn$ independence-oracle calls: since $A$ is a base, $A-i+j$ is a base exactly when it is independent.

A feasible flow of value $k$ defines $Q(i,j)$ as the flow on arc $i\to j$. Flow conservation at the left vertices gives the row equalities, flow conservation at the right vertices gives the column equalities, and the support arcs enforce feasibility. Conversely, every balanced exchange matrix defines such a feasible flow.

Feasibility of the network follows from \Cref{thm:balanced-existence}. Standard feasible-circulation or transportation algorithms therefore return a balanced exchange in time polynomial in the network size. In the real-arithmetic oracle model used in this paper, the demands $x_j$ may be arbitrary reals.
\end{proof}

\begin{remark}
\label{rem:decomposition-sampling}
When an explicit decomposition $x=\sum_{r=1}^m\lambda_r\1_{B_r}$ is available, one may sample from $q_A^x$ without forming the transportation matrix. Draw $B_r$ with probability $\lambda_r$, compute a Brualdi bijection $h_{A,B_r}$, draw $i$ uniformly from $A$, and set $j=h_{A,B_r}(i)$. The resulting exchange has probability
\[
\frac1k\sum_r\lambda_r\ind\{h_{A,B_r}(i)=j\}
=
q_A^x(i,j).
\]
The transportation-flow implementation used by the main algorithm does not require such a decomposition.
\end{remark}

\paragraph{Comparator vertices as a special case.}
If $x=\1_O$ is the incidence vector of a comparator base, choose the one-term base decomposition supported on $O$ in the proof of \Cref{thm:balanced-existence}. The resulting balanced exchange matrix is a single Brualdi bijection from $A$ to $O$. In that case the deletion and insertion marginals are both uniform on their respective bases. For the multilinear extension $F$ of any monotone submodular function and every $\tau\in[0,1]$,
\begin{equation}
\E_{(i,j)\sim q_A^{\1_O}}
\left[\partial_jF(\tau\1_A)-\partial_iF(\tau\1_A)\right]
=
\frac1k\left(
\sum_{o\in O}\partial_oF(\tau\1_A)
-
\sum_{a\in A}\partial_aF(\tau\1_A)
\right).
\label{eq:exact-comparator-gradient}
\end{equation}
Thus every comparator vertex $x=\1_O$ recovers the exact Brualdi identity. Allowing $x$ to range over the convex hull of these vertices exposes a polynomial-dimensional decision variable while preserving the same local benchmark.

Balanced fractional exchanges therefore give an oracle-polynomial way to realize the two marginals of a distribution over comparator-specific Brualdi rules. In particular, comparator vertices recover the exact Brualdi identity, while arbitrary fractional bases interpolate these rules without representing a distribution over bases.

\section{The Balanced-Exchange Bandit Algorithm}
\label{sec:algorithm}

The first subsection presents the bandit algorithm and its linear-loss formulation.  \Cref{sec:estimator-properties} establishes the guarantees of the bandit estimator.  \Cref{sec:omd} proves the resulting regret bounds.

\subsection{Algorithm Overview}

Each round $t$, our algorithm run a Poisson base walk of \Cref{sec:poisson} from some initial base. The exchange distribution used by algorithm is induced by a fractional base $x_t$, which is maintained and updated by our algorithm reach round. 
Given any current base $A$ at the Poisson process, the balanced matrix $Q_A^{x_t}$ converts this single point into a distribution over feasible exchanges: its row constraints make the deleted element uniform on $A$, while its column constraints make the inserted element have marginal $x_t/k$. Thus the same $x_t$ can be used at every state visited by the Poisson trajectory, and the state itself always remains a base.

The deficit term of the Poisson base walk is expressed through the signed exchange gain $G_f$ from \eqref{eq:loo-gain}. To use this quantity in online learning over $P(\M)$, we express its expectation under a balanced exchange distribution as a linear function of the fractional decision $x$. We first separate $G_f$ into insertion and deletion contributions. Fix a monotone submodular function $f$, its multilinear extension $F$, a base $A$, and a time $\tau\in[0,1]$. For a feasible exchange $(i,j)\in\cE(A)$, define the insertion marginal
\begin{equation}
I_f(A,\tau;i,j)
\deq
\E_{R\sim\tau(A-i)}[f(R+j)-f(R)].
\label{eq:insertion-marginal}
\end{equation}

\begin{lemma}
\label{lem:insertion-deletion}
For every feasible exchange $(i,j)\in\cE(A)$,
\begin{align}
0\le I_f(A,\tau;i,j)&\le1,
\label{eq:I-range}\\
I_f(A,\tau;i,j)&=\partial_jF(\tau\1_{A-i}),
\label{eq:I-derivative}\\
I_f(A,\tau;i,j)&\ge \partial_jF(\tau\1_A),
\label{eq:I-dominates}\\
G_f(A,\tau;i,j)&=I_f(A,\tau;i,j)-\partial_iF(\tau\1_A).
\label{eq:G-decomposition}
\end{align}
\end{lemma}

\begin{proof}
Monotonicity gives $f(R+j)-f(R)\ge0$, and the range $f\in[0,1]$ gives the upper bound in \eqref{eq:I-range}. Identity \eqref{eq:I-derivative} is the derivative formula \eqref{eq:derivative} evaluated at $\tau\1_{A-i}$.

The vectors $\tau\1_{A-i}$ and $\tau\1_A$ differ only in coordinate $i$, with the former no larger than the latter. By diminishing multilinear marginals,
\[
I_f(A,\tau;i,j)
=
\partial_jF(\tau\1_{A-i})
\ge
\partial_jF(\tau\1_A),
\]
which proves \eqref{eq:I-dominates}.

For $R\sim\tau(A-i)$,
\[
\E[f(R+i)-f(R)]
=
\partial_iF(\tau\1_{A-i}).
\]
The derivative $\partial_iF$ is independent of coordinate $i$, so this equals $\partial_iF(\tau\1_A)$. Subtracting this marginal from \eqref{eq:insertion-marginal} proves \eqref{eq:G-decomposition}.
\end{proof}

Fix $x\in P(\M)$ with strictly positive coordinates and a balanced matrix $Q=Q_A^x$. Define the average insertion marginal and the corresponding insertion loss by
\begin{equation}
\overline I_{f,A,\tau,x}(j)
\deq
\frac1{x_j}\sum_{i\in A}Q(i,j)I_f(A,\tau;i,j),
\qquad
\ell_{f,A,\tau,x}(j)
\deq
1-\overline I_{f,A,\tau,x}(j).
\label{eq:local-loss}
\end{equation}

\begin{lemma}
\label{lem:loss-properties}
For every $j\in\U$,
\begin{equation}
0\le \ell_{f,A,\tau,x}(j)\le1,
\qquad
\overline I_{f,A,\tau,x}(j)
\ge \partial_jF(\tau\1_A).
\label{eq:loss-properties}
\end{equation}
\end{lemma}

\begin{proof}
By the column equality, the nonnegative weights $Q(i,j)/x_j$ sum to one. Thus $\overline I(j)$ is a convex combination of insertion marginals, each of which belongs to $[0,1]$ by \eqref{eq:I-range}. This proves the range of $\ell$. Every insertion marginal in column $j$ is at least $\partial_jF(\tau\1_A)$ by \eqref{eq:I-dominates}, so their convex combination has the same lower bound.
\end{proof}

The following theorem bound Poisson deficit with a linear term with the insertion loss being the coefficient.

\begin{theorem}
\label{thm:local-linearization}
Let $q_A^x\deq Q_A^x/k$ be a balanced exchange distribution. For every comparator base $O$,
\begin{equation}
\zeta_{f,O}(A,\tau)
-
\E_{(i,j)\sim q_A^x}[G_f(A,\tau;i,j)]
\le
\frac1k\ip{x-\1_O}{\ell_{f,A,\tau,x}}.
\label{eq:local-linearization}
\end{equation}
\end{theorem}

\begin{proof}
To simplify notation, set $I(i,j)\deq I_f(A,\tau;i,j)$ and $\overline I_j\deq\overline I_{f,A,\tau,x}(j)$. By \eqref{eq:G-decomposition},
\begin{align}
\E_{(i,j)\sim q_A^x}[G_f(A,\tau;i,j)]
&=
\frac1k\sum_{i\in A}\sum_{j\in\U}Q(i,j)
\bigl(I(i,j)-\partial_iF(\tau\1_A)\bigr)
\notag\\
&=
\frac1k\left(
\sum_{j\in\U}\sum_{i\in A}Q(i,j)I(i,j)
-
\sum_{i\in A}\partial_iF(\tau\1_A)
\sum_{j\in\U}Q(i,j)
\right)
\notag\\
&=
\frac1k\left(
\sum_{j\in\U}x_j\overline I_j
-
\sum_{i\in A}\partial_iF(\tau\1_A)
\right).
\label{eq:balanced-cancellation}
\end{align}
The last equality uses both balance conditions: the column sum produces $x_j\overline I_j$, and the row sum is exactly one. Subtracting \eqref{eq:balanced-cancellation} from \eqref{eq:zeta} gives
\begin{align*}
\zeta_{f,O}(A,\tau)-\E[G_f]
&=
\frac1k\left(
\sum_{o\in O}\partial_oF(\tau\1_A)
-
\sum_{j\in\U}x_j\overline I_j
\right)\\
&\le
\frac1k\left(
\sum_{o\in O}\overline I_o
-
\sum_{j\in\U}x_j\overline I_j
\right)\\
&=
\frac1k\ip{\1_O-x}{\overline I}.
\end{align*}
The inequality is Lemma~\ref{lem:loss-properties}. Since both $x$ and $\1_O$ have coordinate sum $k$,
\[
\ip{\1_O-x}{\overline I}
=
\ip{\1_O-x}{\1-\ell}
=
\ip{x-\1_O}{\ell}.
\]
This proves \eqref{eq:local-linearization}.
\end{proof}

Equation \eqref{eq:local-linearization} therefore replaces the comparator-directed exchange rule by a linear loss over $P(\M)$. The bandit update estimates a coordinate of this loss indexed by the inserted element $j$. Since the estimate contains the importance weight $1/x_{t,j}$, we need a uniform positive lower bound on every coordinate. Following the shrinkage technique of \citet{flaxman2005}, we obtain this bound by contracting the polytope toward a point whose coordinates are all positive.

For every $u\in\U$, choose a base $B^{(u)}$ containing $u$, and define $x^\circ \deq \frac1n\sum_{u\in\U}\1_{B^{(u)}}.$
Then $x^\circ\in P(\M)$ and $x_u^\circ\ge\frac1n$ for every $u\in\U$. For $\delta\in(0,1)$, let
\begin{equation}
K_\delta
\deq
(1-\delta)P(\M)+\delta x^\circ
=
\cset{(1-\delta)y+\delta x^\circ:y\in P(\M)}.
\label{eq:shrunken-polytope}
\end{equation}
Every $x\in K_\delta$ belongs to $P(\M)$ and satisfies $x_u\ge\delta/n$ for every $u\in\U$. For a comparator base $O$, define its shrunken representative $u_O^\delta\deq(1-\delta)\1_O+\delta x^\circ\in K_\delta$.

The shrinkage has two roles.  It guarantees $1/x_{t,j}\le n/\delta$ on every exploration round, and it supplies a feasible comparator $u_O^\delta$ for mirror descent.  The latter differs from the base vertex $\1_O$ by only a $\delta$-mixture, so the resulting error is handled as an explicit shrinkage term in the regret proof.  We use the negative-entropy potential $\Phi(x)\deq\sum_{u\in\U}(x_u\log x_u-x_u)$, whose Bregman divergence on the positive orthant is
\begin{equation}
D_\Phi(x,y)
\deq
\sum_{u\in\U}
\left(
x_u\log\frac{x_u}{y_u}-x_u+y_u
\right).
\label{eq:entropy-divergence}
\end{equation}

Entropy is well matched to the one-sparse nonnegative estimates produced below: the mirror-descent stability term is the weighted second moment $\sum_u x_{t,u}\wh L_{t,u}^2$, for which the insertion marginal $x_{t,j}/k$ cancels the inverse probability in the estimator.  This cancellation is the reason the variance depends on $n/k$, rather than on the number of feasible exchanges.

Each round instantiates the Poisson base walk in \Cref{alg:poisson-base-walk} with the balanced exchange distribution $p_s(\cdot\mid A)=Q_A^{x_t}/k$, followed by either exploitation or a single exploration query. Let $\eps,\gamma,\delta\in(0,1)$ denote the Poisson starting time, exploration probability, and shrinkage parameter, and let $\eta>0$ be the learning rate. The nonhomogeneous process has intensity $k/s$ on $[\eps,1]$ and hence mean event count $\mu\deq k\log(1/\eps)$. We cap the realized count at an integer $L$. This does not alter the Poisson analysis; it only makes the per-round number of transportation flows worst-case polynomial. For the analysis, the probability space is extended on overflow rounds to include the corresponding uncapped Poisson trajectory, and the overflow probability is bounded separately.

Conditional on $M=m$, the event times of the process with rate $k/s$ are the order statistics of $m$ independent samples with density
\begin{equation}
\varphi_\eps(s)
\deq
\frac{1}{s\log(1/\eps)},
\qquad s\in[\eps,1].
\label{eq:event-time-density}
\end{equation}
Equivalently, if $U\sim\operatorname{Unif}[0,1]$, then $\eps^{1-U}$ has density \eqref{eq:event-time-density}.

Given $x_t$, a nonoverflow simulation starts from a fixed base $A_0$ and, at every event, samples from the balanced distribution associated with the current base and $x_t$.  Here an event again means one Poisson arrival and its associated feasible base exchange. The simulated trajectory uses only the independence oracle and internal randomness, and makes no value query.  With probability $1-\gamma$, the learner queries the terminal base.  With probability $\gamma$, it selects one event uniformly and uses one feasible query to estimate the cumulative deficit.  The complete procedure is given in \Cref{alg:balanced-bandit}. We next describe the one-query estimator $\wh L_t$ used in its update.

\begin{algorithm}[t]
\small
\caption{Poisson Balanced-Exchange Bandit}
\label{alg:balanced-bandit}
\KwIn{Matroid $\M$; start base $A_0$; point $x^\circ$; parameters $\eps,\gamma,\delta,\eta,L$.}
Set $x_1=x^\circ$ and $\mu=k\log(1/\eps)$.\;
\For{$t=1,\ldots,T$}{
 Sample $M_t\sim\operatorname{Poisson}(\mu)$.\;
 \eIf{$M_t>L$}{
  Play $S_t=A_0$, observe $f_t(S_t)$, and set $\wh L_t=0$.\;
 }{
  Sample and sort event times $\eps<\tau_{t,1}<\cdots<\tau_{t,M_t}\le1$ according to \eqref{eq:event-time-density}.\;
  Set $A_{t,0}=A_0$.\;
  \For{$\event=1,\ldots,M_t$}{
   Compute a balanced exchange $Q_{A_{t,\event-1}}^{x_t}$.\;
   Draw $(i_{t,\event},j_{t,\event})\sim Q_{A_{t,\event-1}}^{x_t}/k$.\;
   Set $A_{t,\event}=A_{t,\event-1}-i_{t,\event}+j_{t,\event}$.\;
  }
  Draw $Z_t\sim\operatorname{Bernoulli}(\gamma)$.\;
  \eIf{$Z_t=0$}{
   Play the final base $S_t=A_{t,M_t}$, observe $f_t(S_t)$, and set $\wh L_t=0$.\;
  }{
   \eIf{$M_t=0$}{
    Play $S_t=A_0$, observe $f_t(S_t)$, and set $\wh L_t=0$.\;
   }{
    Draw $J_t\sim\operatorname{Unif}\{1,\ldots,M_t\}$.\;
    Put $A\deq A_{t,J_t-1}$, $\tau\deq\tau_{t,J_t}$, and $(i,j)\deq(i_{t,J_t},j_{t,J_t})$.\;
    Sample $R\sim\tau(A-i)$ and $C_t\sim\operatorname{Bernoulli}(1/2)$.\;
    \eIf{$C_t=1$}{
     Play $S_t=R+j$, observe $y_t=f_t(S_t)$, and set $z_t=2(1-y_t)$.\;
    }{
     Play $S_t=R$, observe $y_t=f_t(S_t)$, and set $z_t=2y_t$.\;
    }
    Set
    $\displaystyle
      \wh L_t
      =
      \frac{M_t}{\gamma}
      \frac{w(\tau)z_t}{x_{t,j}}e_j,
    $
    where $e_j$ is the $j$th coordinate vector and $w(s)\deq s e^{-(1-s)}$.\;
   }
  }
 }
 Set $\wt x_{t+1,u}=x_{t,u}\exp(-\eta\wh L_{t,u})$ for every $u\in\U$.\;
 Compute the entropy projection
 $\displaystyle
 x_{t+1}\in\argmin_{x\in K_\delta}D_\Phi(x,\wt x_{t+1}).
 $.\;
}
\end{algorithm}

\paragraph{Bandit estimator.} To analyze an exploration round, condition on the selected pre-event base $A$, time $\tau$, and exchange $(i,j)$.  The leave-one-out set $R\sim\tau(A-i)$ makes both $R$ and $R+j$ feasible.  A fair coin chooses which of these two sets is queried, and the observation is transformed into
$2(1-f_t(R+j))$ or $2f_t(R)$.  The conditional mean of this statistic is $1-I_{f_t}(A,\tau;i,j)$, the insertion loss produced by the local linearization. The estimate is nonnegative and bounded by two, which gives the multiplicative entropy update a clean local-norm analysis.

To state the loss being estimated, let
\begin{equation}
\ell_{t,\event}
\deq
\ell_{f_t,A_{t,\event-1},\tau_{t,\event},x_t}
\in[0,1]^{\U}
\label{eq:event-loss-vector}
\end{equation}
be the local loss vector at event $\event$.  By the balanced-exchange linearization, the inner product of $\ell_{t,\event}$ with $x_t-\1_O$ controls the local deficit relative to every comparator base $O$.  We aggregate these vectors with the discount $w(s)\deq s e^{-(1-s)}$ and define
\begin{equation}
L_t
\deq
\frac{\ind\{M_t\le L\}}{k}
\sum_{\event=1}^{M_t}
w(\tau_{t,\event})\ell_{t,\event}.
\label{eq:aggregate-loss}
\end{equation}
On an exploration round with $1\le M_t\le L$, we use the following estimator
\begin{equation}
\wh L_t
=
\frac{M_t}{\gamma}
\frac{w(\tau_{t,J_t})z_t}{x_{t,j_{t,J_t}}}e_{j_{t,J_t}},
\label{eq:one-sparse-loss-estimator}
\end{equation}
for $L_t$, and set $\wh L_t=0$ when $M_t>L$. This is a useful bandit estimator in the precise senses needed by the regret analysis: it uses one feasible query, is conditionally unbiased for $L_t$, and has controlled weighted second moment. These properties are proved in \Cref{sec:estimator-properties}.

After the query, the learner performs the standard exponential weights update and returns to $K_\delta$ by the exact entropy projection displayed in \Cref{alg:balanced-bandit}.  \emph{All vector divisions and logarithms in the algorithm are coordinatewise.}  A coordinate with a large estimated insertion loss is downweighted, while the projection restores the matroid base constraints and the uniform positivity condition. The resulting Bregman projection inequality is the projection property used by the online-learning analysis.

\paragraph{Computational complexity.}
We will set the event cap to be $L=O(k\log(1/\eps)+\log T)$ in the following sections, and each balanced exchange is computed by a polynomial-size transportation problem. The exact entropy projection is a separable strictly convex minimization problem over a shrunken matroid base polytope and is oracle-polynomial in the real arithmetic model by \citet{suehiro2012}. Thus the per-round computation is oracle-polynomial and no projection tolerance is needed.

\subsection{Properties of the Bandit Estimator}
\label{sec:estimator-properties}

This subsection establishes the estimator guarantees used in the regret analysis. The estimator uses one feasible value query, is conditionally unbiased for $L_t$, and has controlled weighted second moment.
\begin{proposition}
\label{prop:feasible-query}
Every query made by \Cref{alg:balanced-bandit} is an independent set, and each round makes exactly one value query.
\end{proposition}

\begin{proof}
On an overflow round the algorithm plays the base $A_0$. On a nonoverflow exploitation round it plays the final state of a sequence of feasible base exchanges, hence a base. On an exploration round, $R\subseteq A-i\subseteq A$, so $R$ is independent. Also
$R+j\subseteq A-i+j$, and $A-i+j$ is a base because the sampled exchange lies in the support of a balanced matrix. Thus both possible queried sets are independent. Every branch plays one set and observes one scalar value.
\end{proof}

Let $\cH_t$ be the sigma-field generated by the interaction history through round $t-1$, the predictable state $x_t$, and $f_t$, before the fresh round-$t$ randomization. Thus $f_t$ is fixed under conditioning on $\cH_t$ but remains unobserved by the learner.

\begin{lemma}
\label{lem:one-query-statistic}
Fix a base $A$, a time $\tau$, and a feasible exchange $(i,j)$. Independently draw $R\sim\tau(A-i)$ and $C\sim\operatorname{Bernoulli}(1/2)$, and define
\begin{equation}
Z\deq
\begin{cases}
2(1-f(R+j)),& C=1,\\
2f(R),& C=0.
\end{cases}
\label{eq:one-query-statistic}
\end{equation}
Then $0\le Z\le2$ and $\E[Z\mid i,j,A,\tau]=1-I_f(A,\tau;i,j)$.
\end{lemma}

\begin{proof}
The range follows from $f\in[0,1]$. Averaging over the fair coin and then over $R$ gives $\E[Z\mid i,j,A,\tau]=\E_{R\sim\tau(A-i)}[1-f(R+j)+f(R)]=1-I_f(A,\tau;i,j)$.
\end{proof}

\begin{lemma}
\label{lem:estimator-unbiased}
For every round $t$ of \Cref{alg:balanced-bandit}, it holds that $\E[\wh L_t\mid\cH_t]=\E[L_t\mid\cH_t]$, where $L_t$ is defined in \eqref{eq:aggregate-loss} and $\wh L_t$ is defined in \eqref{eq:one-sparse-loss-estimator}.
\end{lemma}

\begin{proof}
Fix a realized nonoverflow pre-event base $A$ and time $\tau$ before sampling the exchange, and let $Q\deq Q_A^{x_t}$. Conditional on $(i,j)$, \Cref{lem:one-query-statistic} gives $\E[z_t\mid i,j,A,\tau]=1-I_{f_t}(A,\tau;i,j)$. The $j$th coordinate of the expected importance-weighted vector is $\frac{w(\tau)}{k}\bigl(1-\frac1{x_{t,j}}\sum_{i\in A}Q(i,j)I_{f_t}(A,\tau;i,j)\bigr)$, where we used $\sum_iQ(i,j)=x_{t,j}$. Hence $\E[(w(\tau)z_t/x_{t,j})e_j\mid A,\tau]=(w(\tau)/k)\ell_{f_t,A,\tau,x_t}$.

Now condition on the nonoverflow trajectory. For $M_t>0$, uniform sampling of $J_t$ and multiplication by $M_t$ recover the sum over events, while $\E[Z_t/\gamma]=1$ corrects for exploration. Therefore $\E[\wh L_t]=(1/k)\sum_{\event=1}^{M_t}w(\tau_{t,\event})\ell_{t,\event}$ conditional on the trajectory. When $M_t=0$ or $M_t>L$, both sides are zero. Taking the tower expectation proves the lemma.
\end{proof}

\begin{lemma}
\label{lem:estimator-second-moment}
For every round $t$, $\E[\sum_{u\in\U}x_{t,u}\wh L_{t,u}^2\mid\cH_t]\le4n(\mu^2+\mu)/(k\gamma)$.
\end{lemma}

\begin{proof}
The estimator is zero unless the round is a nonoverflow exploration round with at least one event. On such a round, one-sparsity gives $\sum_{u\in\U}x_{t,u}\wh L_{t,u}^2=(Z_tM_t^2/\gamma^2)w(\tau)^2z_t^2/x_{t,j}$, where $Z_t$ is the exploration indicator.

Fix the selected pre-event base $A$ and time $\tau$ before sampling its exchange. Since $x_t\in K_\delta$, every coordinate $x_{t,u}$ is positive. Moreover, $w(\tau)\le1$ for $\tau\in[\eps,1]$, and \Cref{lem:one-query-statistic} gives $z_t^2\le4$. Expanding the conditional expectation over the entering element and using its marginal distribution from \eqref{eq:balanced-marginals}, we obtain
\begin{align*}
\E\left[
\frac{w(\tau)^2z_t^2}{x_{t,j}}
\,\middle|\,
A,\tau
\right]
&=
\sum_{u\in\U}
\Prb(j=u\mid A,\tau)
\E\left[
\frac{w(\tau)^2z_t^2}{x_{t,u}}
\,\middle|\,
A,\tau,j=u
\right]\\
&\le
\sum_{u\in\U}
\frac{x_{t,u}}{k}\frac{4}{x_{t,u}}
=
\frac{4n}{k}.
\end{align*}
Thus the sampling probability $x_{t,u}/k$ cancels the remaining inverse weight $1/x_{t,u}$ for every possible entering element $u$. Since the bound is uniform over the realized $A$ and $\tau$, it also holds after conditioning on the preceding trajectory.

Conditional on $M_t=m\in\{1,\ldots,L\}$, the uniform event index averages this bound over the $m$ event states and times, while $\E[Z_t]/\gamma^2=1/\gamma$. Thus $\E[\sum_ux_{t,u}\wh L_{t,u}^2\mid\cH_t,M_t=m]\le4nm^2/(k\gamma)$. Averaging over the Poisson count and using $\E[M_t^2]=\mu^2+\mu$ proves the claim; the case $m=0$ contributes zero.
\end{proof}

\subsection{Regret Analysis}
\label{sec:omd}

Now, we analyze the regret of the algorithm.
We use the following standard negative-entropy Online Mirror Descent framework, see \citet[Chapter~6]{orabona2026}. We record the short specialization needed here because it applies pathwise to the adaptive loss sequence.

\begin{lemma}
\label{lem:entropy-omd}
Let $K\subset(0,\infty)^n$ be a convex compact set. Let $x_1,\ldots,x_{T+1}\in K$, let $g_1,\ldots,g_T\in[0,\infty)^n$, and define
$
\wt x_{t+1,u}\deq x_{t,u}e^{-\eta g_{t,u}}.
$
Suppose that, for every $t=1,\ldots,T$, $x_{t+1}$ is the exact entropy projection
\begin{equation}
x_{t+1}
\in
\argmin_{x\in K}
D_\Phi(x,\wt x_{t+1}).
\label{eq:entropy-projection-update}
\end{equation}
Then, for every comparator $u\in K$,
\begin{equation}
\sum_{t=1}^T\ip{x_t-u}{g_t}
\le
\frac{D_\Phi(u,x_1)}{\eta}
+
\frac{\eta}{2}
\sum_{t=1}^T\sum_{j=1}^n x_{t,j}g_{t,j}^2.
\label{eq:entropy-omd-regret}
\end{equation}
\end{lemma}

\begin{proof}
Fix a round $t$. The Bregman Pythagorean inequality for the exact projection in \eqref{eq:entropy-projection-update} and the negative-entropy dual update give
\begin{align*}
D_\Phi(u,x_{t+1})
&\le
D_\Phi(u,\wt x_{t+1})\\
&=
\sum_{j=1}^n
\left(
u_j\log\frac{u_j}{x_{t,j}e^{-\eta g_{t,j}}}
-u_j
+x_{t,j}e^{-\eta g_{t,j}}
\right)\\
&=
\sum_{j=1}^n
\left(
u_j\log\frac{u_j}{x_{t,j}}
-u_j
+x_{t,j}
\right)
+\eta\sum_{j=1}^n u_jg_{t,j}
+\sum_{j=1}^n x_{t,j}
\left(e^{-\eta g_{t,j}}-1\right)\\
&=
D_\Phi(u,x_t)
+\eta\ip{u}{g_t}
+\sum_{j=1}^n x_{t,j}
\left(e^{-\eta g_{t,j}}-1\right).
\end{align*}
The second equality substitutes $\wt x_{t+1,j}=x_{t,j}e^{-\eta g_{t,j}}$ into \eqref{eq:entropy-divergence}. The third uses the coordinatewise identity
\begin{equation*}
\log\frac{u_j}{x_{t,j}e^{-\eta g_{t,j}}}
=
\log\frac{u_j}{x_{t,j}}
+\eta g_{t,j},
\end{equation*}
and adds and subtracts $x_{t,j}$, so that the first sum is exactly $D_\Phi(u,x_t)$.
Since $g_t\ge0$ and $e^{-a}\le1-a+a^2/2$ for $a\ge0$, it follows that
\begin{align*}
\ip{x_t-u}{g_t}
&\le
\frac{D_\Phi(u,x_t)-D_\Phi(u,x_{t+1})}{\eta}
+
\frac\eta2\sum_jx_{t,j}g_{t,j}^2.
\end{align*}
Summing over $t$ telescopes the divergence terms. Dropping the final nonnegative divergence proves \eqref{eq:entropy-omd-regret}.
\end{proof}

\begin{lemma}
\label{lem:comparator-divergence}
For every base $O$, it holds that 
$D_\Phi(u_O^\delta,x^\circ)
\le
k\log n
\le
k\log(en).$
\end{lemma}

\begin{proof}
Both $u_O^\delta$ and $x^\circ$ belong to the base polytope, so their coordinate sums are $k$. The linear terms in \eqref{eq:entropy-divergence} cancel:
\[
D_\Phi(u_O^\delta,x^\circ)
=
\sum_j u_{O,j}^\delta
\log\frac{u_{O,j}^\delta}{x_j^\circ}.
\]
Every coordinate of a base-polytope point is at most one, while $x_j^\circ\ge1/n$. Hence
\[
\log\frac{u_{O,j}^\delta}{x_j^\circ}
\le
\log\frac1{x_j^\circ}
\le
\log n
\]
whenever $u_{O,j}^\delta>0$. Summing and using $\sum_ju_{O,j}^\delta=k$ proves the first inequality. The second is immediate.
\end{proof}

\begin{proposition}
\label{prop:expected-linear-regret}
Let $L_t$ be the aggregate loss in \eqref{eq:aggregate-loss}. For every fixed base $O$,
\begin{equation}
\E\left[
\sum_{t=1}^T\ip{x_t-u_O^\delta}{L_t}
\right]
\le
\frac{k\log(en)}{\eta}
+
\frac{2\eta n(\mu^2+\mu)T}{k\gamma}.
\label{eq:expected-linear-regret}
\end{equation}
\end{proposition}

\begin{proof}
Apply \Cref{lem:entropy-omd} pathwise with $K=K_\delta$, $g_t=\wh L_t$, $x_1=x^\circ$, and comparator $u=u_O^\delta$. This gives
\begin{align}
\sum_{t=1}^T\ip{x_t-u}{\wh L_t}
&\le
\frac{D_\Phi(u,x^\circ)}{\eta}
+
\frac\eta2\sum_{t=1}^T\sum_jx_{t,j}\wh L_{t,j}^2.
\label{eq:pathwise-linear-regret}
\end{align}
Because $x_t$ and $u$ are $\cH_t$-measurable, \Cref{lem:estimator-unbiased} implies
\[
\E\ip{x_t-u}{\wh L_t}
=
\E\ip{x_t-u}{L_t}.
\]
Take expectations in \eqref{eq:pathwise-linear-regret}, use \Cref{lem:comparator-divergence} for the initial divergence, and use \Cref{lem:estimator-second-moment} for each quadratic term. The coefficient is
\[
\frac\eta2\cdot\frac{4n(\mu^2+\mu)}{k\gamma}
=
\frac{2\eta n(\mu^2+\mu)}{k\gamma},
\]
which proves \eqref{eq:expected-linear-regret} after summing over $T$ rounds.
\end{proof}

The final regret analysis combines three inequalities: the Poisson theorem bounds approximation regret by event deficits, balanced exchanges bound those deficits by linear regret, and entropic mirror descent controls the resulting adaptive losses. On overflow rounds, we extend the probability space to include the uncapped Poisson trajectory and bound the discrepancy caused by the implementation's event cap.

\begin{theorem}
\label{thm:optimized-regret}
Let $T\ge2$, set $\eps\deq1/T$ and $\mu\deq k\log T$.
If $n(\mu^2+\mu)\log(en)<T$, run \Cref{alg:balanced-bandit} with
\[
L
\deq
\left\lceil2e(\mu+\log T)\right\rceil,
\qquad
\delta
\deq
\min\cset{\frac12,\frac1{\mu T}},
\]
and
\[
\gamma
\deq
\left(
\frac{n(\mu^2+\mu)\log(en)}{T}
\right)^{1/3},
\qquad
\eta
\deq
\sqrt{
\frac{k^2\gamma\log(en)}{2n(\mu^2+\mu)T}
}.
\]
Then
\begin{equation}
\begin{aligned}
R_{1-1/e}(T)
&=
O\!\left(
\left[
n\bigl((k\log T)^2+k\log T\bigr)\log(en)
\right]^{1/3}T^{2/3}
+1
\right)
\\
&=
\widetilde O\!\left(n^{1/3}k^{2/3}T^{2/3}\right).
\end{aligned}
\label{eq:optimized-regret}
\end{equation}
If $n(\mu^2+\mu)\log(en)\ge T$, the same bound is achieved by playing any fixed base.
\end{theorem}

\begin{proof}
Let
\[
\alpha\deq1-1/e,
\qquad
\alpha_\eps\deq(1-\eps)\alpha.
\]
Let
\begin{equation}
O^\star\in\argmax_{O\in\B}\sum_{t=1}^Tf_t(O).
\label{eq:best-base}
\end{equation}
Because the adversary is oblivious, $O^\star$ is fixed before the learner's random choices.

\paragraph{Step 1: restore the uncapped Poisson process.}
At the beginning of round $t$, the algorithm samples the full Poisson count $M_t$. If $M_t\le L$, it also samples all event times and exchanges. If $M_t>L$, the implementation stops immediately. For analysis, extend the probability space on such a round by additionally sampling the event times and the balanced-exchange trajectory that the uncapped process would have generated using the same predictable point $x_t$. On a nonoverflow round this trajectory is identical to the one generated by the algorithm. Let $A_t^{\mathrm{fin}}$ be its final base, and let
\[
(A_{t,\event-1},\tau_{t,\event},e_{t,\event})
\qquad(\event=1,\ldots,M_t)
\]
denote its pre-event states, event times, and exchanges.

Condition on the history before round $t$. The function $f_t$ and point $x_t$ are fixed under this conditioning, and the extended trajectory is exactly \Cref{alg:poisson-base-walk} with $p_s(\cdot\mid A)=Q_A^{x_t}/k$. Apply \Cref{thm:general-poisson} with comparator $O^\star$, and then sum over rounds. We obtain
\begin{align}
\alpha_\eps\sum_{t=1}^Tf_t(O^\star)
-
\E\left[\sum_{t=1}^Tf_t(A_t^{\mathrm{fin}})\right]
&\le
\E\left[
\sum_{t=1}^T\sum_{\event=1}^{M_t}D_{t,\event}
\right],
\label{eq:poisson-summed}
\end{align}
where
\begin{align}
D_{t,\event}
&\deq
w(\tau_{t,\event}) \Bigl( \zeta_{f_t,O^\star}(A_{t,\event-1},\tau_{t,\event}) - G_{f_t}(A_{t,\event-1},\tau_{t,\event};e_{t,\event}) \Bigr).
\label{eq:round-event-deficit}
\end{align}

\paragraph{Step 2: charge event deficits to linear regret.}
Fix an event of the extended trajectory and condition on its pre-event history. The current base, time, and learner point are fixed, and the exchange is drawn from the balanced distribution $q_A^{x_t}$. By \Cref{thm:local-linearization},
\begin{align}
\E[D_{t,\event}\mid\text{pre-event history}]
&\le
\frac{w(\tau_{t,\event})}{k}
\ip{x_t-\1_{O^\star}}{\ell_{t,\event}}.
\label{eq:conditional-deficit-linearization}
\end{align}
Here $\ell_{t,\event}$ is also defined on overflow rounds through the extended trajectory, even though the implementation does not generate it.

On a nonoverflow round, use
$x_t-\1_{O^\star}=(x_t-u_{O^\star}^\delta)+\delta(x^\circ-\1_{O^\star})$ and the definition \eqref{eq:aggregate-loss}:
\begin{align}
&\frac1k\sum_{\event=1}^{M_t}w(\tau_{t,\event})
\ip{x_t-\1_{O^\star}}{\ell_{t,\event}}
=
\ip{x_t-u_{O^\star}^\delta}{L_t}
+
\frac{\delta}{k}
\sum_{\event=1}^{M_t}w(\tau_{t,\event})
\ip{x^\circ-\1_{O^\star}}{\ell_{t,\event}}.
\label{eq:shrinkage-decomposition}
\end{align}
Because $\ell_{t,\event}\ge0$, $w\le1$, and $\sum_jx_j^\circ=k$,
\[
\ip{x^\circ-\1_{O^\star}}{\ell_{t,\event}}
\le
\ip{x^\circ}{\ell_{t,\event}}
\le
k.
\]
Thus the second term in \eqref{eq:shrinkage-decomposition} is at most $\delta M_t$.

On an overflow round, we do not charge the event losses to online learning. Instead, for each event,
\begin{align}
\frac{w(\tau)}{k}\ip{x_t-\1_{O^\star}}{\ell}
&\le
\frac1k\ip{x_t}{\ell}
\le
\frac1k\sum_jx_{t,j}
=1.
\label{eq:overflow-event-bound}
\end{align}
The overflow moment has the exact Poisson representation
\begin{align*}
\E[M\ind\{M>L\}]
&=
\sum_{m=L+1}^{\infty}m e^{-\mu}\frac{\mu^m}{m!}
=
\mu\sum_{r=L}^{\infty}e^{-\mu}\frac{\mu^r}{r!}
=
\mu\Prb(M\ge L),
\end{align*}
where the middle equality changes variables to $r=m-1$. Combining \eqref{eq:conditional-deficit-linearization}--\eqref{eq:overflow-event-bound}, taking expectations, and using
$\E[M_t\ind\{M_t\le L\}]\le\mu$ therefore gives
\begin{align}
\E\left[
\sum_{t=1}^T\sum_{\event=1}^{M_t}D_{t,\event}
\right]
&\le
\E\left[
\sum_{t=1}^T\ip{x_t-u_{O^\star}^\delta}{L_t}
\right]
+
\delta\mu T
+
\mu T\Prb(M\ge L).
\label{eq:total-deficit-bound}
\end{align}

\paragraph{Step 3: complete the comparison with the implemented rewards.}
Apply \Cref{prop:expected-linear-regret} with $O=O^\star$ to the first term on the right-hand side of \eqref{eq:total-deficit-bound}. This controls the cumulative event deficits by the online linear-regret bound, together with the shrinkage and overflow terms already displayed there. It remains to replace the terminal reward $f_t(A_t^{\mathrm{fin}})$ of the extended uncapped trajectory in \eqref{eq:poisson-summed} by the reward $f_t(S_t)$ of the set actually played. On a nonoverflow exploitation round, $S_t=A_t^{\mathrm{fin}}$. On a nonoverflow exploration round, $f_t(S_t)\ge0$ while $f_t(A_t^{\mathrm{fin}})\le1$, so the discrepancy is at most one. On an overflow round, the algorithm plays $A_0$ while $A_t^{\mathrm{fin}}$ exists only in the extended trajectory, and the same range bound again gives a discrepancy of at most one. These three cases yield, pathwise,
\begin{equation}
f_t(S_t)
\ge
f_t(A_t^{\mathrm{fin}})
-
\ind\{M_t>L\}
-
\ind\{M_t\le L\ \text{and round $t$ explores}\}.
\label{eq:actual-versus-terminal}
\end{equation}
The second indicator has expectation at most $\gamma$, since a round explores with probability $\gamma$. Taking expectations and summing gives
\begin{equation}
\E\left[\sum_{t=1}^Tf_t(S_t)\right]
\ge
\E\left[\sum_{t=1}^Tf_t(A_t^{\mathrm{fin}})\right]
-T\Prb(M>L)-
\gamma T.
\label{eq:actual-versus-terminal-expectation}
\end{equation}

Combining \eqref{eq:poisson-summed}, \eqref{eq:total-deficit-bound}, \Cref{prop:expected-linear-regret}, and \eqref{eq:actual-versus-terminal-expectation}, and using $\Prb(M>L)\le\Prb(M\ge L)$, yields
\begin{align}
R_{\alpha_\eps}(T)
&\le
\left[
\gamma
+\delta\mu
+(1+\mu)\Prb(M\ge L)
\right]T
+
\frac{k\log(en)}{\eta}
+
\frac{2\eta n(\mu^2+\mu)T}{k\gamma}.
\label{eq:parameterized-regret}
\end{align}

It remains to substitute the parameters in the theorem statement. For every $\theta>0$, the exponential-moment bound gives
\[
\Prb(M\ge L)
\le
\exp\bigl(\mu(e^\theta-1)-\theta L\bigr).
\]
Since $L\ge2e\mu$, choosing $e^\theta=L/\mu$ yields
\[
\Prb(M\ge L)
\le
\left(\frac{e\mu}{L}\right)^L
\le
2^{-L}.
\]
With $c\deq2e\log2>1$ and the prescribed value of $L$, it follows that
\[
(1+\mu)T\Prb(M\ge L)
\le
(1+\mu)e^{-c\mu}T^{1-c}
=O(1),
\]
where the implicit constant is universal because $(1+\mu)e^{-c\mu}$ is bounded on $[0,\infty)$ and $T^{1-c}\le2^{1-c}$. Moreover, $\delta\mu T\le1$. Finally, the prescribed learning rate satisfies
\begin{align*}
\frac{k\log(en)}{\eta}
+
\frac{2\eta n(\mu^2+\mu)T}{k\gamma}
&=
2\sqrt{
\frac{2n(\mu^2+\mu)T\log(en)}{\gamma}
}
=
2\sqrt2\,
\left[n(\mu^2+\mu)\log(en)\right]^{1/3}T^{2/3},
\end{align*}
whereas
$\gamma T=\left[n(\mu^2+\mu)\log(en)\right]^{1/3}T^{2/3}$.
For every reward sequence,
\begin{align*}
R_\alpha(T)
&=
R_{\alpha_\eps}(T)
+
\eps\alpha\max_{O\in\B}\sum_{t=1}^Tf_t(O)\\
&\le
R_{\alpha_\eps}(T)+\eps T.
\end{align*}
Since $\eps=1/T$ and $\mu=k\log T$, the conversion cost is at most one. Combining this inequality with \eqref{eq:parameterized-regret} and the preceding estimates proves \eqref{eq:optimized-regret} in the nontrivial regime.

If $n(\mu^2+\mu)\log(en)\ge T$, then
$\left[n(\mu^2+\mu)\log(en)\right]^{1/3}T^{2/3}\ge T$.
Since the comparator reward is at most $T$ and the reward of a fixed base is nonnegative, a fixed-base strategy satisfies $R_\alpha(T)\le T$, proving \eqref{eq:optimized-regret} in the remaining case.
\end{proof}

\section*{Acknowledgement}
Generative LLM was used as an interactive aid in the developing of the proofs. It also helps with the organization and drafting of the proofs, which are verified by the authors. The idea and technical framework are also proposed by the human authors.

\bibliographystyle{plainnat}
\bibliography{balanced_exchanges_arxiv}

@article{auer2002nonstochastic,
  author  = {Peter Auer and Nicol{\`o} Cesa-Bianchi and Yoav Freund and Robert E. Schapire},
  title   = {The Nonstochastic Multiarmed Bandit Problem},
  journal = {SIAM Journal on Computing},
  volume  = {32},
  number  = {1},
  pages   = {48--77},
  year    = {2002},
  doi     = {10.1137/S0097539701398375}
}

@article{brualdi1969,
  author  = {Richard A. Brualdi},
  title   = {Comments on Bases in Dependence Structures},
  journal = {Bulletin of the Australian Mathematical Society},
  volume  = {1},
  number  = {2},
  pages   = {161--167},
  year    = {1969}
}

@inproceedings{buchbinderfeldman2024,
  author    = {Niv Buchbinder and Moran Feldman},
  title     = {Deterministic Algorithm and Faster Algorithm for Submodular Maximization Subject to a Matroid Constraint},
  booktitle = {Proceedings of the 65th IEEE Annual Symposium on Foundations of Computer Science},
  pages     = {700--712},
  year      = {2024}
}

@article{calinescu2011,
  author  = {Gruia C{\u{a}}linescu and Chandra Chekuri and Martin P{\'a}l and Jan Vondr{\'a}k},
  title   = {Maximizing a Monotone Submodular Function Subject to a Matroid Constraint},
  journal = {SIAM Journal on Computing},
  volume  = {40},
  number  = {6},
  pages   = {1740--1766},
  year    = {2011}
}

@inproceedings{cardoso2019,
  author    = {Adrian Rivera Cardoso and Rachel Cummings},
  title     = {Differentially Private Online Submodular Minimization},
  booktitle = {Proceedings of the Twenty-Second International Conference on Artificial Intelligence and Statistics},
  series    = {Proceedings of Machine Learning Research},
  volume    = {89},
  pages     = {1650--1658},
  year      = {2019}
}

@inproceedings{chekuri2010,
  author    = {Chandra Chekuri and Jan Vondr{\'a}k and Rico Zenklusen},
  title     = {Dependent Randomized Rounding via Exchange Properties of Combinatorial Structures},
  booktitle = {Proceedings of the 51st IEEE Annual Symposium on Foundations of Computer Science},
  pages     = {575--584},
  year      = {2010}
}

@inproceedings{chen2017interactive,
  author    = {Lin Chen and Andreas Krause and Amin Karbasi},
  title     = {Interactive Submodular Bandit},
  booktitle = {Advances in Neural Information Processing Systems 30},
  pages     = {141--152},
  year      = {2017}
}

@article{filmusward2014,
  author  = {Yuval Filmus and Justin Ward},
  title   = {Monotone Submodular Maximization over a Matroid via Non-Oblivious Local Search},
  journal = {SIAM Journal on Computing},
  volume  = {43},
  number  = {2},
  pages   = {514--542},
  year    = {2014}
}

@inproceedings{flaxman2005,
  author    = {Abraham D. Flaxman and Adam Tauman Kalai and H. Brendan McMahan},
  title     = {Online Convex Optimization in the Bandit Setting: Gradient Descent without a Gradient},
  booktitle = {Proceedings of the Sixteenth Annual ACM--SIAM Symposium on Discrete Algorithms},
  pages     = {385--394},
  year      = {2005}
}

@inproceedings{fourati2023,
  author    = {Fares Fourati and Vaneet Aggarwal and Christopher John Quinn and Mohamed-Slim Alouini},
  title     = {Randomized Greedy Learning for Non-Monotone Stochastic Submodular Maximization under Full-Bandit Feedback},
  booktitle = {Proceedings of the 26th International Conference on Artificial Intelligence and Statistics},
  series    = {Proceedings of Machine Learning Research},
  volume    = {206},
  pages     = {7455--7471},
  year      = {2023}
}

@inproceedings{fourati2024stochasticgreedy,
  author    = {Fares Fourati and Christopher John Quinn and Mohamed-Slim Alouini and Vaneet Aggarwal},
  title     = {Combinatorial Stochastic-Greedy Bandit},
  booktitle = {Proceedings of the AAAI Conference on Artificial Intelligence},
  volume    = {38},
  pages     = {12052--12060},
  year      = {2024}
}

@inproceedings{ganz2026poisson,
  author    = {Ganz Rozenman, Amit and Ariel Kulik and Roy Schwartz and Mohit Singh},
  title     = {A {Poisson} Process for Submodular Maximization},
  booktitle = {Proceedings of the 58th Annual ACM Symposium on Theory of Computing},
  pages     = {2152--2163},
  year      = {2026}
}

@article{hazan2012,
  author  = {Elad Hazan and Satyen Kale},
  title   = {Online Submodular Minimization},
  journal = {Journal of Machine Learning Research},
  volume  = {13},
  pages   = {2903--2922},
  year    = {2012}
}

@inproceedings{ito2022,
  author    = {Shinji Ito},
  title     = {Revisiting Online Submodular Minimization: Gap-Dependent Regret Bounds, Best of Both Worlds and Adversarial Robustness},
  booktitle = {Proceedings of the 39th International Conference on Machine Learning},
  series    = {Proceedings of Machine Learning Research},
  volume    = {162},
  pages     = {9678--9694},
  year      = {2022}
}

@inproceedings{kempe2003,
  author    = {David Kempe and Jon Kleinberg and {\'E}va Tardos},
  title     = {Maximizing the Spread of Influence through a Social Network},
  booktitle = {Proceedings of the Ninth ACM SIGKDD International Conference on Knowledge Discovery and Data Mining},
  pages     = {137--146},
  year      = {2003},
  doi       = {10.1145/956750.956769}
}

@article{krause2008,
  author  = {Andreas Krause and Ajit Singh and Carlos Guestrin},
  title   = {Near-Optimal Sensor Placements in {Gaussian} Processes: Theory, Efficient Algorithms and Empirical Studies},
  journal = {Journal of Machine Learning Research},
  volume  = {9},
  pages   = {235--284},
  year    = {2008}
}

@article{nemhauser1978,
  author  = {George L. Nemhauser and Laurence A. Wolsey and Marshall L. Fisher},
  title   = {An Analysis of Approximations for Maximizing Submodular Set Functions---{I}},
  journal = {Mathematical Programming},
  volume  = {14},
  pages   = {265--294},
  year    = {1978}
}

@inproceedings{nie2022,
  author    = {Guanyu Nie and Mridul Agarwal and Abhishek Kumar Umrawal and Vaneet Aggarwal and Christopher John Quinn},
  title     = {An Explore-Then-Commit Algorithm for Submodular Maximization under Full-Bandit Feedback},
  booktitle = {Proceedings of the 38th Conference on Uncertainty in Artificial Intelligence},
  series    = {Proceedings of Machine Learning Research},
  volume    = {180},
  pages     = {1541--1551},
  year      = {2022}
}

@inproceedings{nie2023framework,
  author    = {Guanyu Nie and Yididiya Y. Nadew and Yanhui Zhu and Vaneet Aggarwal and Christopher John Quinn},
  title     = {A Framework for Adapting Offline Algorithms to Solve Combinatorial Multi-Armed Bandit Problems with Bandit Feedback},
  booktitle = {Proceedings of the 40th International Conference on Machine Learning},
  series    = {Proceedings of Machine Learning Research},
  volume    = {202},
  pages     = {26166--26198},
  year      = {2023}
}

@inproceedings{nie2025ksubmodular,
  author    = {Guanyu Nie and Vaneet Aggarwal and Christopher John Quinn},
  title     = {Stochastic {$k$}-Submodular Bandits with Full Bandit Feedback},
  booktitle = {Proceedings of the 24th International Conference on Autonomous Agents and Multiagent Systems},
  pages     = {2696--2698},
  year      = {2025}
}

@article{niazadeh2023,
  author  = {Rad Niazadeh and Negin Golrezaei and Joshua R. Wang and Fransisca Susan and Ashwinkumar Badanidiyuru},
  title   = {Online Learning via Offline Greedy Algorithms: Applications in Market Design and Optimization},
  journal = {Management Science},
  volume  = {69},
  number  = {7},
  pages   = {3797--3817},
  year    = {2023}
}

@misc{orabona2026,
  author = {Francesco Orabona},
  title  = {Online Learning: A Modern Introduction Using Convex Optimization},
  year   = {2026},
  note   = {arXiv:1912.13213, version 10}
}

@inproceedings{sisalem2024,
  author    = {Si Salem, Tareq and G{\"o}zde {\"O}zcan and Iasonas Nikolaou and Evimaria Terzi and Stratis Ioannidis},
  title     = {Online Submodular Maximization via Online Convex Optimization},
  booktitle = {Proceedings of the 38th AAAI Conference on Artificial Intelligence},
  volume    = {38},
  pages     = {15038--15046},
  year      = {2024}
}

@inproceedings{streeter2008,
  author    = {Matthew Streeter and Daniel Golovin},
  title     = {An Online Algorithm for Maximizing Submodular Functions},
  booktitle = {Advances in Neural Information Processing Systems 21},
  pages     = {1577--1584},
  year      = {2008}
}

@inproceedings{streeter2009,
  author    = {Matthew Streeter and Daniel Golovin and Andreas Krause},
  title     = {Online Learning of Assignments},
  booktitle = {Advances in Neural Information Processing Systems 22},
  pages     = {1794--1802},
  year      = {2009}
}

@inproceedings{suehiro2012,
  author    = {Daiki Suehiro and Kohei Hatano and Shuji Kijima and Eiji Takimoto and Kiyohito Nagano},
  title     = {Online Prediction under Submodular Constraints},
  booktitle = {Proceedings of the 23rd International Conference on Algorithmic Learning Theory},
  series    = {Lecture Notes in Computer Science},
  volume    = {7568},
  pages     = {260--274},
  publisher = {Springer},
  year      = {2012},
  doi       = {10.1007/978-3-642-34106-9_22}
}

@inproceedings{tajdini2024,
  author    = {Artin Tajdini and Lalit Jain and Kevin Jamieson},
  title     = {Nearly Minimax Optimal Submodular Maximization with Bandit Feedback},
  booktitle = {Advances in Neural Information Processing Systems 37},
  year      = {2024},
  doi       = {10.52202/079017-3051}
}

@inproceedings{vondrak2008,
  author    = {Jan Vondr{\'a}k},
  title     = {Optimal Approximation for the Submodular Welfare Problem in the Value Oracle Model},
  booktitle = {Proceedings of the 40th Annual ACM Symposium on Theory of Computing},
  pages     = {67--74},
  year      = {2008}
}

@inproceedings{wan2023bandit,
  author    = {Zongqi Wan and Jialin Zhang and Wei Chen and Xiaoming Sun and Zhijie Zhang},
  title     = {Bandit Multi-Linear {DR}-Submodular Maximization and Its Applications on Adversarial Submodular Bandits},
  booktitle = {Proceedings of the 40th International Conference on Machine Learning},
  series    = {Proceedings of Machine Learning Research},
  volume    = {202},
  pages     = {35491--35524},
  year      = {2023}
}

@inproceedings{yue2011,
  author    = {Yisong Yue and Carlos Guestrin},
  title     = {Linear Submodular Bandits and Their Application to Diversified Retrieval},
  booktitle = {Advances in Neural Information Processing Systems 24},
  pages     = {2483--2491},
  year      = {2011}
}

@inproceedings{zhang2019,
  author    = {Mingrui Zhang and Lin Chen and Hamed Hassani and Amin Karbasi},
  title     = {Online Continuous Submodular Maximization: From Full-Information to Bandit Feedback},
  booktitle = {Advances in Neural Information Processing Systems 32},
  pages     = {9210--9220},
  year      = {2019}
}

\appendix

\section{Sampling the Nonhomogeneous Poisson Process}
\label{app:poisson-sampling}

We record the count and conditional event-time distribution used in \Cref{alg:balanced-bandit}.

\begin{proposition}
\label{prop:poisson-sampling-law}
Let $N$ be a nonhomogeneous Poisson process on $[\eps,1]$ with intensity $\lambda(s)\deq k/s$. Then
\begin{equation}
M\deq N([\eps,1])
\sim
\operatorname{Poisson}(\mu),
\qquad
\mu\deq k\log(1/\eps).
\label{eq:appendix-poisson-count}
\end{equation}
Conditional on $M=m$, the unordered event times are independent with density
\begin{equation}
\frac{\lambda(s)}{\mu}
=
\frac1{s\log(1/\eps)}
\qquad(s\in[\eps,1]),
\label{eq:appendix-event-density}
\end{equation}
and the ordered event times are their order statistics. If $U\sim\operatorname{Unif}[0,1]$, then $S\deq\eps^{1-U}$ has density \eqref{eq:appendix-event-density}.
\end{proposition}

\begin{proof}
The integrated intensity is
\[
\Lambda(a,b)
\deq
\int_a^b\frac{k}{s}\,ds
=
k\log(b/a).
\]
By the defining increment law of a nonhomogeneous Poisson process,
$N([a,b])$ is Poisson with mean $\Lambda(a,b)$. Taking $a=\eps$ and $b=1$ gives \eqref{eq:appendix-poisson-count}.

The joint density of exactly $m$ ordered events at
$\eps<t_1<\cdots<t_m\le1$ is
\begin{equation}
e^{-\mu}\prod_{r=1}^m\lambda(t_r).
\label{eq:ordered-joint-density}
\end{equation}
Indeed, partition the interval into small cells around the $t_r$'s and their complement. The probability of one event in each selected cell and none elsewhere is the product of the corresponding first-order Poisson probabilities; dividing by the cell widths and taking the limit gives \eqref{eq:ordered-joint-density}. Integrating over the ordered simplex yields
\[
e^{-\mu}\frac1{m!}
\left(\int_\eps^1\lambda(s)\,ds\right)^m
=
e^{-\mu}\frac{\mu^m}{m!},
\]
which is the probability that $M=m$. Dividing \eqref{eq:ordered-joint-density} by this probability gives the conditional ordered density
\[
m!\prod_{r=1}^m\frac{\lambda(t_r)}{\mu}.
\]
This is exactly the density of the order statistics of $m$ independent samples from $\lambda/\mu$, proving \eqref{eq:appendix-event-density}.

Finally, for $s\in[\eps,1]$,
\begin{align*}
\Prb(\eps^{1-U}\le s)
&=
\Prb\left(U\le1-\frac{\log s}{\log\eps}\right)\\
&=
\frac{\log(s/\eps)}{\log(1/\eps)}.
\end{align*}
Differentiating gives $1/(s\log(1/\eps))$.
\end{proof}

\end{document}